\def\year{2018}\relax
\documentclass[letterpaper]{article} 
\usepackage{calc}
\usepackage{aaai18}  
\usepackage{times}  
\usepackage{helvet}  
\usepackage{courier}  
\usepackage{url}  
\usepackage{graphicx}  

\usepackage{mathrsfs}
\usepackage{amsfonts}
\usepackage{amsmath, amssymb,amsthm}
\usepackage[shortlabels]{enumitem}
\usepackage{calc}
\newtheorem{theorem}{Theorem}[section]
\newtheorem{corollary}[theorem]{Corollary}
\newtheorem{lemma}[theorem]{Lemma}

\newtheorem{remark}[theorem]{Remark}

\newcommand{\ds}{\textnormal{d}s}
\newcommand{\Kp}{K_p}
\newcommand{\Kb}{K_b}
\newcommand{\ExP}{\mathbb{E}}
\newcommand{\lmax}{\lambda_{\max}}
\newcommand{\lmin}{\lambda_{\min}}

\newcommand{\bE}{\mathbb{E}}
\newcommand{\cF}{\mathcal{F}}

\newcommand{\dReal}{\mathbb{R}^d}
\newcommand{\Real}{\mathbb{R}}

\newcommand{\thS}{\theta^*}

\newcommand{\st}{\alpha}
\newcommand{\df}{\mathrm{d}}

\newcommand{\Id}{\mathbb{I}}

\newcommand{\bart}{\bar{\theta}}

\newcommand{\tI}[1]{t_{#1}}
\newcommand{\tr}{\top}

\newcommand{\cE}{\mathcal{E}}

\newcommand{\cA}{\mathcal{A}}
\newcommand{\Ro}{R_{\mathrm{wc}}}

\newcommand{\ei}{\epsilon}

\newcommand{\vop}{VoP }
\newcommand{\cS}{C_*}
\newcommand{\cR}{\mathcal{R}}

\newcommand{\eDb}{E^\text{d}_{[n_0,n+1]}}
\newcommand{\eDbp}{E^\text{d}_{[n',n+1]}}
\newcommand{\eMb}{E^\text{m}_{[n_0,n+1]}}
\newcommand{\Emid}{\cE_{n_0, n_1}^{\textnormal{mid}}}
\newcommand{\Eaft}{\cE_{n_0, n_1}^{\textnormal{after}}}

\newcommand{\sA}{\mathscr{A}}
\newcommand{\calS}{\mathcal{S}}
\newcommand{\Kl}{K_\lambda\,}
\newcommand{\Km}{K_m\,}

\newcommand{\cMb}{C_{m2}}
\newcommand{\Gnd}{G_{n_0, n}}
\newcommand{\nMid}{n_\textnormal{c}}

\usepackage{cancel}
\usepackage{xcolor}
\definecolor{OliveGreen}{rgb}{0,0.6,0}

\newcommand{\RNum}[1]{\uppercase\expandafter{\romannumeral #1\relax}}

\begin{document}
%
\title{Finite Sample Analyses for TD(0) with Function Approximation}
\author{
  Gal Dalal\thanks{Equal contribution.} \\
  	Technion, Israel \\
	gald@campus.technion.ac.il\\
  \And
  Bal\'azs Sz\"or\'enyi$^*$ \\
  Oath (formerly Yahoo) Research\\
	szorenyi.balazs@gmail.com \\
  \And
  Gugan Thoppe$^*$ \\
  Duke University, USA \\
	gugan.thoppe@gmail.com
  \And
  Shie Mannor \\
  Technion, Israel\\
	shie@ee.technion.ac.il
 }

\maketitle

\begin{abstract}
TD(0) is one of the most commonly used algorithms in reinforcement learning. Despite this, there is no existing finite sample analysis for TD(0) with function approximation, even for the linear case. Our work is the first to provide such results. Existing convergence rates for Temporal Difference (TD) methods apply only to somewhat modified versions, e.g., projected variants or ones where stepsizes depend on unknown problem parameters. Our analyses obviate these artificial alterations by exploiting strong properties of TD(0). We provide convergence rates both in expectation and with high-probability. The two are obtained via different approaches that use relatively unknown, recently developed stochastic approximation techniques.
\end{abstract}

\section{Introduction}
\label{sec:intro}



Temporal Difference (TD) algorithms lie at the core of Reinforcement Learning (RL), dominated by the celebrated TD(0) algorithm. The term has been coined in \cite{SuttonBarto1998book}, describing an iterative process of updating an estimate of a value function $V^\pi(s)$ with respect to a given policy $\pi$ based on temporally-successive samples. The classical version of the algorithm uses a tabular representation, i.e., entry-wise storage of the value estimate per each state $s\in {\cal S}$. However, in many problems, the state-space ${\cal S}$ is too large for such a vanilla approach. The common practice to mitigate this caveat is to approximate the value function using some parameterized family. Often, linear regression is used, i.e., $V^\pi(s)\approx\theta^\top\phi(s)$. This allows for an efficient implementation of TD(0) even on large state-spaces and has shown to perform well in a variety of problems \cite{tesauro1995temporal,powell2007approximate}. More recently, TD(0) has become prominent in many state-of-the-art RL solutions when combined with deep neural network architectures, as an integral part of fitted value iteration \cite{mnih2015human,silver2016mastering}. In this work we focus on the former case of linear Function Approximation (FA); nevertheless, we consider this work as a preliminary milestone in route to achieving theoretical guarantees for non-linear RL architectures.

Two types of convergence rate results exist in literature: in expectation and with high probability. We stress that no results of either type exist for the actual, commonly used, TD(0) algorithm with linear FA; our work is the first to provide such results. In fact, it is the first work to give a convergence rate for an unaltered online TD algorithm of any type. We emphasize that TD(0) with linear FA is formulated and used with non-problem-specific stepsizes. Also, it does not require a projection step to keep $\theta$ in a `nice' set. In contrast, the few recent works that managed to provide convergence rates for TD(0) analyzed only altered versions of them. These modifications include a projection step and eigenvalue-dependent stepsizes, or they apply only to the average of iterates; we expand on this in the coming section. 

\subsection{Existing Literature}
The first TD(0) convergence result was obtained by \cite{tsitsiklis1997analysis} for both finite and infinite state-spaces. Following that, a key result by \cite{borkar2000ode} paved the path to a unified and convenient tool for convergence analyses of Stochastic Approximation (SA), and hence of TD algorithms. This tool is based on the Ordinary Differential Equation (ODE) method.  Essentially, that work showed that under the right conditions, the SA trajectory follows the solution of a suitable ODE, often referred to as its \emph{limiting ODE}; thus, it eventually converges to the solution of the limiting ODE. Several usages of this tool in RL literature can be found in \cite{sutton2009convergent,sutton2009fast,sutton2015emphatic}.

As opposed to the case of asymptotic convergence analysis of TD algorithms, very little is known about their finite sample behavior. We now briefly discuss the few existing results on this topic. In \cite{borkar2008stochastic}, a concentration bound is given for generic SA algorithms. Recent works \cite{kamal2010convergence,thoppe2015concentration} obtain better concentration bounds via tighter analyses. The results in these works are conditioned on the event that the $n_0-$th iterate lies in some a-priori chosen bounded region containing the desired equilibria; this, therefore, is the caveat in applying them to TD(0).

In \cite{korda2015td}, convergence rates for TD(0) with mixing-time consideration have been given. We note that even though doubts were recently raised regarding the correctness results there \cite{LaSze17:Rep}, we shall treat them as correct for the sake of discussion. The results in \cite{korda2015td} require the learning rate to be set based on prior knowledge about system dynamics, which, as argued in the paper, is problematic; alternatively, they apply to the average of iterates. Additionally, unlike in our work, a strong requirement for all high probability bounds  is that the iterates need to lie in some a-priori chosen bounded set; this is ensured there via projections (personal communication). In similar spirit, results  for TD(0) requiring prior knowledge about system parameters are also given in \cite{kondathesis}. An additional work by \cite{liu2015finite} considered the  gradient TD algorithms GTD(0) and GTD2, which were first introduced in \cite{sutton2009fast,sutton2009convergent}. That work interpreted the algorithms as gradient methods to some saddle-point optimization problem. This enabled them to obtain convergence rates on altered versions of these algorithms using results from the convex optimization literature. Despite the alternate approach, in a similar fashion to the results above, a projection step that keeps the parameter vectors in a convex set is needed there.

Bounds similar in flavor to ours are also given in \cite{frikha2012,fathi2013}. However, they apply only to a class of SA methods satisfying strong assumptions, which do not hold for  TD(0). In particular, neither the uniformly Lipschitz assumption nor its weakened version, the Lyapunov Stability-Domination criteria, hold for TD(0) when formulated in their iid noise setup. 

Three additional works \cite{yu2009convergence,lazaric2010finite,pan2017accelerated} provide sample complexity bounds on the batch LSTD algorithms. However, in the context of finite sample analysis, these belong to a different class of algorithms. The case of online TD learning has proved to be more practical, at the expense of increased analysis difficulty compared to LSTD methods.

\subsection{Our Contributions}
Our work is the first to give bounds on the convergence rate of TD(0) in its \emph{original, unaltered form}. In fact, it is the first to obtain convergence rate results for an unaltered online TD algorithm of any type.	
Indeed, as discussed earlier, existing convergence rates apply only to online TD algorithms with alterations such as projections and stepsizes dependent on unknown problem parameters; alternatively, they only apply to average of iterates.

The methodologies for obtaining the expectation and high probability bounds are quite different. The former has a short and elegant proof that follows via induction using a subtle trick from \cite{kamal2010convergence}. This bound applies to a general family of stepsizes that is not restricted to square-summable sequences, as usually was required by most previous works. This result reveals an explicit interplay between the stepsizes and noise. 

As for the key ingredients in proving our high-probability bound, we first show that the $n$-th iterate at worst is only $O(n)$ away from the solution $\thS$. Based on that, we then utilize tailor-made stochastic approximation tools to show that after some additional steps all subsequent iterates are $\ei$-close to the solution w.h.p. This novel analysis approach allows us to obviate the common alterations mentioned above.
Our key insight regards the role of the driving matrix's smallest eigenvalue $\lambda$. The convergence rate is dictated by it when it is below some threshold; for larger values, the rate is dictated by the noise.

We believe these two analysis approaches are not limited to TD(0) alone. 


\section{Problem Setup}
We consider the problem of policy evaluation for a Markov Decision Process (MDP). A MDP is defined by the 5-tuple $(\calS, \sA,P,\cR,\gamma)$ \cite{sutton1988learning}, where $\calS$ is the set of states, $\sA$ is the set of actions, $P = P(s'|s,a)$ is the transition kernel, $\cR(s,a,s')$ is the reward function, and $\gamma\in(0,1)$ is the discount factor. In each time-step, the process is in some state $s \in \calS$, an action $a \in \sA$ is taken, the system transitions to a next state $s'\in {\cal S}$ according to the transition kernel $P$, and an immediate reward $r$ is received according to $\cR(s,a,s')$. Let policy $\pi:{\cal S} \rightarrow \sA$ be a stationary mapping from states to actions. Assuming the associated Markov chain is ergodic and uni-chain, let $\nu$ be the induced stationary distribution. Moreover, let $V^\pi(s)$ be the value function at state $s$ w.r.t. $\pi$ defined via the Bellman equation $V^\pi(s) = \mathbb{E}_\nu[r+\gamma V^\pi(s')]$. In our policy evaluation setting, the goal is to estimate $V^\pi(s)$  using linear regression, i.e., $V^\pi(s)\approx\theta^\top\phi(s)$, where $\phi(s) \in \mathbb{R}^d$ is a feature vector at state $s$, and $\theta \in \mathbb{R}^d$ is a weight vector. For brevity, we omit the notation $\pi$ and denote $\phi(s),~\phi(s')$  by $\phi,~\phi'$.

Let $\{(\phi_n,\phi_n',r_n)\}_n$ be iid samples of $(\phi,\phi',r)$.\footnote{The iid assumption does not hold in practice; however, it is standard when dealing with convergence bounds in reinforcement learning \cite{liu2015finite,sutton2009convergent,sutton2009fast}. It allows for sophisticated and well-developed techniques from SA theory, and it is not clear how it can be avoided. Indeed, the few papers that obviate this assumption assume other strong properties such as exponentially-fast mixing time \cite{korda2015td,tsitsiklis1997analysis}. In practice, drawing samples from the stationary distribution is often simulated by taking the last sample from a long trajectory, even though knowing when to stop the trajectory is again a hard theoretical problem.
	Additionally, most recent implementations of TD algorithms use long replay buffers that shuffle samples. This reduces the correlation between the samples, thereby making our assumption more realistic.} Then the TD(0) algorithm has the update rule
\begin{equation}
\label{eqn:TD_Method}
\theta_{n + 1} = \theta_n + \alpha_n [r_n + \gamma \phi_n'^\tr \theta_n - \phi_n^\tr \theta_n] \phi_n,
\end{equation}
where $\alpha_n$ is the stepsize. For analysis, we can rewrite the above as
\begin{equation}\label{eq:SA_traj}
\theta_{n + 1} = \theta_n + \alpha_n [h(\theta_n) + M_{n + 1}] \enspace,
\end{equation}
where $h(\theta) = b - A \theta$ and
\begin{equation}
\label{Defn:MartDiffNoise}
M_{n+1}=\left(r_n +\gamma\phi_n'^\top\theta_n-\phi_n^\top\theta_n\right)\phi_n - [b -A \theta_n]
\enspace,
\end{equation}
with $A  = \bE_\nu[\phi(\phi-\gamma\phi')^\top]$ and $b = \bE_\nu[r\phi].$ It is known that $A$ is positive definite  \cite{bertsekas2012dynamic} and that  \eqref{eq:SA_traj} converges to $\thS := A^{-1}b$ \cite{borkar2008stochastic}. Note that
\begin{equation}
\label{Defn:h}
h(\theta) = -A [\theta - \thS] \enspace.
\end{equation}

\section{Main Results}

Our first main result is a bound on the expected decay rate of the TD(0) iterates. It requires the following assumption. 
	{\renewcommand*\theenumi{$\pmb{\cA_\arabic{enumi}}$}
		\begin{enumerate}[leftmargin=4ex+\labelsep]
			\item \label{assum:bounded second moments} 
			For some $K_s >0$,
			$$\ExP[\|M_{n+1}\|^2|{\cal F}_n] \leq K_s [1 + \|\theta_n - \thS\|^2].$$
			 
		\end{enumerate}
	}
This assumption follows from \eqref{Defn:MartDiffNoise} when, for example, $\{(\phi_n,\phi_n',r_n)\}_n$ have uniformly bounded second moments. The latter is a common assumption in such results; e.g., \cite{sutton2009fast,sutton2009convergent}. 
	
Recall that all eigenvalues of a symmetric matrix are real. For a symmetric matrix $X,$ let $\lambda_{\min}(X)$ and $\lambda_{\max}(X)$ be its minimum and maximum eigenvalues, respectively.

\begin{theorem}[Expected Decay Rate for TD(0)]
	\label{thm:ExPDecayRate}
	Fix $\sigma \in (0,1)$ and let $\alpha_n = (n + 1)^{-\sigma}.$ Fix $\lambda \in (0, \lmin(A + A^\tr)).$ Then, under \ref{assum:bounded second moments}, for $n \geq 1,$
	\[
	\ExP\|\theta_n - \theta^*\|^2 \leq K_1  e^{- (\lambda/2) n^{1 - \sigma}}  +  \frac{K_2}{n^\sigma},
	\]
	where $K_1, K_2 \geq 0$ are some constants that depend on both $\lambda$ and $\sigma;$ see \eqref{eq:K1} and \eqref{eq:K2} for the exact expressions.
\end{theorem}
\begin{remark}[Stepsize tradeoff -- \RNum{1}]
	The exponentially decaying term in Theorem~\ref{thm:ExPDecayRate} corresponds to the convergence rate of the noiseless TD(0) algorithm, while the inverse polynomial term appears due to the martingale noise $M_n.$ The inverse impact of $\sigma$ on these two terms introduces the following tradeoff:
	\begin{enumerate}
		\item For $\sigma$ close to $0,$ which corresponds to  slowly  decaying stepsizes, the first term converges faster. This stems from speeding up the underlying noiseless TD(0) process.

		\item For $\sigma$ close to $1,$ which corresponds to  quickly  decaying stepsizes, the second term converges faster. This is due to better mitigation of the martingale noise; recall that $M_{n+1}$ is scaled with $\st_n.$
	\end{enumerate}
	While this insight is folklore, a formal estimate of the tradeoff, to the best of our knowledge, has been obtained here for the first time.
\end{remark}

\begin{remark}[Stepsize tradeoff -- \RNum{2}]
	A practitioner might expect initially large stepsizes to speed up convergence. However, Theorem~\ref{thm:ExPDecayRate} shows that as $\sigma$ becomes small, the convergence rate starts being dominated by the martingale difference noise; i.e., choosing a larger stepsize will help speed up convergence only up to some threshold. 
\end{remark}

\begin{remark}[Non square-summable stepsizes]
	\label{remark: squre summability}
	In Theorem~\ref{thm:ExPDecayRate}, unlike most works, $\sum_{n \geq 0} \st_n^2$ need not be finite. Thus this result is applicable for a wider class of stepsizes; e.g., $1/n^\kappa$ with $\kappa \in (0,  1/2].$ In \cite{borkar2008stochastic}, on which much of the existing RL literature is based on, the square summability assumption is due to the Gronwall inequality. In contrast, in our work, we use the  Variation of Parameters Formula \cite{lakshmikantham1998method} for comparing the SA trajectory to appropriate trajectories of the limiting ODE; it is a stronger tool than Gronwall inequality.
\end{remark}

Our second main result is a high-probability bound for a specific stepsize. It requires the following assumption.
{\renewcommand*\theenumi{$\pmb{\cA_\arabic{enumi}}$}
	\begin{enumerate}[leftmargin=4ex+\labelsep]
		\setcounter{enumi}{1}
		\item \label{assum:bounded_feat} All rewards $r(s, a, s^\prime)$ and feature vectors $\phi(s)$ are uniformly bounded, i.e., $\|\phi(s)\| \leq 1/2, \;  \forall s \in \calS,$ and  $|r(s, a, s^\prime)| \leq 1,$ $\forall s, s^\prime \in \calS,$ $a \in \sA.$
	\end{enumerate}
}
This assumption is well accepted in the literature \cite{liu2015finite,korda2015td}.

In the following results, the $\tilde{O}$ notation hides problem dependent constants and poly-logarithmic terms.

\begin{theorem}[TD(0) Concentration Bound]
	\label{thm: convergence rate}
	Let %
	$
	\lambda \in (0,\min_{i \in [d]} \{\mathrm{real}(\lambda_i(A))\}),
	$
	where $\lambda_i(A)$ is the $i$-th eigenvalue of $A$.
	Let $\st_n = (n+1)^{-1}$.
	Then, under \ref{assum:bounded_feat},  for $\epsilon>0$ and $\delta \in (0,1),$
	there exists a function
	\begin{align*}
	N(\epsilon,&\delta)
    \\&=\tilde{O}\left(\max\left\{\left[\frac{1}{\ei}\right]^{1+ \frac{1}{\lambda}} \left[\ln\frac{1}{\delta}\right]^{1 + \frac{1}{\lambda}}, \left[\frac{1}{\ei}\right]^{2} \left[\ln \frac{1}{\delta}\right]^{3}\right\}\right)
	\end{align*}
	such that
	\[
	\Pr\left\{ \|\theta_n-\thS\| \leq \epsilon  ~\forall n \geq N(\epsilon,\delta) \right\} \geq 1-\delta
	\enspace.
	\]
\end{theorem}



To enable direct comparison with previous works, one can obtain a following weaker implication of Theorem~\ref{thm: convergence rate} by dropping quantifier $\forall$ inside the event. This translates to the following.

\begin{theorem}
	\label{thm:AlternateMain}[TD(0) High-Probability Convergence Rate]
	Let  $\lambda$ and  $\st_n$ be as in Theorem~\ref{thm: convergence rate}. Fix $\delta \in (0, 1).$  Then, under \ref{assum:bounded_feat}, there exists some function $N_0(\delta) = O(\ln(1/\delta))$ such that for all $n \geq N_0(\delta),$
	\[
	\Pr\left\{\|\theta_n - \thS\| = \tilde{O}\left(n^{-\min\{1/2, \lambda/ (\lambda + 1)\}}\right) \right\} \geq 1 - \delta.
	\]
\end{theorem}

\begin{proof}
Fix some $n$, and choose $\epsilon = \epsilon(n)$ so that $n = N(\epsilon,\delta)$.
Then, on one hand,
$1-\delta \leq \Pr\{\|\theta_n-\thS\| \leq \epsilon\}$ due to Theorem~\ref{thm: convergence rate} and, on the other hand,
$\epsilon = \tilde{O}\left(n^{-\min\{1/2, \lambda/ (\lambda + 1)\}}\right)$ by the definition of $N(\epsilon,\delta)$.
The claimed result follows.
\end{proof}


%
\begin{remark}[Eigenvalue dependence]
	Theorem~\ref{thm:AlternateMain} shows that the rate improves as $\lambda$ increases from $0$ to $1;$ however, beyond $1$ it remains fixed at $1/\sqrt{n}.$ As seen in the proof of Theorem~\ref{thm: convergence rate}, this is because the rate is dictated by noise when $\lambda > 1,$  and by the limiting ODE when $\lambda < 1.$
\end{remark}

\begin{remark}[Comparison to \cite{korda2015td}]
	Recently, doubts were raised in \cite{LaSze17:Rep} regarding the correctness of the results in \cite{korda2015td}. Nevertheless, given the current form of those results, the following discussion is in order.
		
	The expectation bound in Theorem 1, \cite{korda2015td} requires the TD(0) stepsize to satisfy $\alpha_n = f_n(\lambda)$ for some function $f_n,$ where $\lambda$ is as above. Theorem~2 there obviates this, but it applies to the average of iterates. In contrast, our expectation bound does not need any scaling of the above kind and applies directly to the TD(0) iterates. Moreover, our result applies to a broader family of stepsizes; see Remark~\ref{remark: squre summability}. Our expectation bound when compared to that of Theorem~2, \cite{korda2015td} is of the same order (even though theirs is for the average of iterates).
	As for the high-probability concentration bounds in Theorems 1\&2, \cite{korda2015td}, they require projecting the iterates to some bounded set (personal communication). In contrast, our result applies directly to the original TD(0) algorithm  and  we obviate all the above modifications.
\end{remark}

\section{Proof of Theorem~\ref{thm:ExPDecayRate}}
We begin with an outline of our proof for Theorem~\ref{thm:ExPDecayRate}. Our first key step is to identify a ``nice" Liapunov function $V(\theta)$. Then, we apply conditional expectation to eliminate the linear noise terms in the relation between $V(\theta_n)$ and $V(\theta_{n + 1});$ this subtle trick appeared in \cite{kamal2010convergence}. Lastly, we use induction to obtain desired result.

Our first two results hold for stepsize sequences of generic form. All that we require for $\{\alpha_n\}$ is to satisfy $\sum_{n \geq 0} \alpha_n = \infty,$ $\lim_{n \to \infty} \alpha_n = 0$ and $\sup_{n \geq 0} \alpha_n \leq 1.$  


Notice that the matrices $(A^\tr + A)$ and $(A^\tr A + K_s I)$ are symmetric, where $K_s$ is the constant from \ref{assum:bounded second moments}. Further, as $A$ is positive definite, the above matrices are also positive definite. Hence their minimum and maximum eigenvalues are strictly positive. This is used in the proofs in this section.

\begin{lemma}
	\label{lem:ProdMatBd}
	For $n \geq 0,$ let $\lambda_n := \lmax(\Lambda_n),$ where
	\[
	\Lambda_n := \Id - \st_n (A + A^\tr) + \st_n^2 (A^\tr A + K_s I).
	\]
	Fix $\lambda \in (0, \lmin(A + A^\tr)).$ Let $m$ be so that $\forall k \geq m,~\st_k\leq  \frac{\lmin(A + A^\tr)- \lambda}{\lmax(A^\tr A + K_s I)}.$  Then for any $k, n$ such that $n \geq k \geq 0,$
	\[
	\prod_{i = k}^n \lambda_i \leq \Kp e^{-\lambda [\sum_{i = k}^{n} \alpha_i]}\enspace,
	\]
	where
	\[
	\Kp :=\max_{\ell_1 \leq \ell_2 \leq m} \prod_{\ell=\ell_1}^{\ell_2} e^{\alpha_\ell (\mu + \lambda)}\enspace,
	\]
	with $\mu = -\lmin(A + A^\tr) + \lmax(A^\tr A + K_s I).$
\end{lemma}

\begin{proof}
	Using Weyl's inequality, we have
	\begin{equation}
	\label{eqn:InequalityWeyl}
	\lambda_n \leq \lmax(\Id - \st_n (A + A^\tr)) + \alpha_n^2 \lmax(A^\tr A + K_s I).
	\end{equation}
	Since $\lmax(\Id - \st_n (A + A^\tr)) \leq (1 - \alpha_n \lmin(A + A^\tr)),$ we have
	\[
	\lambda_n \leq e^{[-\alpha_n \lmin(A^\tr + A) + \alpha_n^2 \lmax(A^\tr A + K_s I)]}.
	\]
	For $n < m,$ using $\alpha_n \leq 1$ and hence $\alpha_n^2 \leq \alpha_n,$ we have the following weak bound:
	\begin{equation} \label{eq:low indices bound}
	\lambda_n \leq e^{\alpha_n \mu}.
	\end{equation}
	On the other hand, for $n \geq m,$ we have
	\begin{align} \label{eq:large indices bound}
	\lambda_n 
	&\leq e^{- \lambda \alpha_n} e^{- \alpha_n [ (\lmin(A^\tr + A) - \lambda) - \alpha_n \lmax(A^\tr A + K_s I)]} 
	\notag
	\\
	&\leq e^{- \lambda \alpha_n}.
	\end{align}
	To prove the desired result, we consider three cases: $k \leq n \leq m,$ $m \leq k \leq n$ and $k \leq m \leq n.$ For the last case, using \eqref{eq:low indices bound} and \eqref{eq:large indices bound}, we have
	\begin{align*}
	\prod_{\ell = k}^n \lambda_\ell
	&\leq  \left[\prod_{\ell = k}^{m} \lambda_\ell\right] e^{-\lambda(\sum_{\ell = m+1}^{n} \alpha_\ell)}
	\\
	&=  \left[\prod_{\ell = k}^{m} \lambda_\ell\right] e^{\lambda(\sum_{\ell = k}^{m} \alpha_\ell)} e^{-\lambda(\sum_{\ell = k}^{n} \alpha_\ell)} 		
	\\
	&\leq \Kp e^{-\lambda(\sum_{\ell = k}^{n} \alpha_\ell)} \enspace,
	\end{align*}
	as desired. Similarly, it can be shown that bound holds in other cases as well. The desired result thus follows.
\end{proof}

Using Lemma~\ref{lem:ProdMatBd}, we now prove a convergence rate in expectation for general stepsizes.

\begin{theorem}[Technical Result: Expectation Bound]
	\label{thm:ExP_Bound_Main_Technical_Result}
	Fix $\lambda \in (0, \lmin(A + A^\tr)).$ Then, under \ref{assum:bounded second moments},
	\begin{align*}
	\ExP\|\theta_{n + 1} - \theta^*\|^2 \leq& \Kp \left[e^{-\lambda \sum_{k = 0}^{n} \alpha_k}\right]\ExP\|\theta_0 - \theta^*\|^2 \\
	&+ K_s \Kp \sum_{i = 0}^{n} \left[e^{-\lambda \sum_{k = i + 1}^{n} \alpha_k} \right] \alpha_i^2,
	\end{align*}
	where $\Kp, K_s \geq 0$ are constants as defined in Lemmas~\ref{lem:ProdMatBd} and \ref{assum:bounded second moments}, respectively.
\end{theorem}

\begin{proof}
	Let $V(\theta) = \|\theta-\theta^*\|^2.$ Using \eqref{eq:SA_traj} and \eqref{Defn:h}, we have
	\[
	\theta_{n + 1} - \thS = (I - \alpha_n A) (\theta_n - \thS) + \alpha_n M_{n + 1}.
	\]
	Hence
	\begin{align*}
	V&(\theta_{n+1})
	\\  = &  (\theta_{n + 1} - \theta^*)^\top (\theta_{n + 1} - \theta^*)\\
	= &[(I - \alpha_nA) (\theta_n - \thS) + \alpha_n M_{n + 1}]^\tr \\
	  &\times[(I - \alpha_n A) (\theta_n - \thS) + \alpha_n M_{n + 1}]\\
	= & (\theta_n - \thS)^\tr [I - \alpha_n(A^\tr + A) + \alpha_n^2 A^\tr A] (\theta_n - \thS)\\
	&  + \alpha_n (\theta_n - \thS)^\tr (I - \alpha_n A)^\tr M_{n + 1} 
	\\
	&+ \alpha_n M_{n + 1}^\tr (I - \alpha_n A) (\theta_n - \thS) +  \alpha_n^2 \|M_{n+1}\|^2.
	\end{align*}
	Taking conditional expectation and using $\ExP[M_{n + 1}|\cF_n] = 0,$ we get
	\begin{align*}
	\mathbb{E}[&V(\theta_{n+1})|{\cal F}_n]   
	= 
	\alpha_n^2 \mathbb{E}[\|M_{n+1}\|^2|{\cal F}_n]
	\\
	&+
	(\theta_n - \thS)^\tr [I - \alpha_n(A^\tr + A) + \alpha_n^2 A^\tr A] (\theta_n - \thS)  
	.
	\end{align*}
	Therefore, using \ref{assum:bounded second moments},
	\[
	\mathbb{E}[V(\theta_{n+1})|{\cal F}_n]   \leq (\theta_n - \thS)^\tr \Lambda_n(\theta_n - \thS)  + K_s \alpha_n^2,
	\]
	where $\Lambda_n =  [I - \alpha_n(A^\tr + A) + \alpha_n^2 (A^\tr A + K_s I)] .$ Since $\Lambda_n$ is a symmetric matrix, all its eigenvalues are real. With $\lambda_n := \lambda_{\max}(\Lambda_n),$ we have
	\begin{equation*}
	\mathbb{E}[V(\theta_{n+1})|{\cal F}_n]   \leq \lambda_n V(\theta_n)  + K_s \alpha_n^2.
	\end{equation*}
	Taking expectation on both sides and letting $w_n = \ExP[V(\theta_n)],$ we have
	\[
	w_{n + 1} \leq \lambda_n w_n + K_s \alpha_n^2.
	\]
	Sequentially using the above inequality, we have
	\[
	w_{n + 1} \leq \left[\prod_{k = 0}^{n} \lambda_k\right] w_0 + K_s \sum_{i = 0}^{n} \left[\prod_{k = i + 1}^{n} \lambda_k\right] \alpha_i^2.
	\]
	Using Lemma~\ref{lem:ProdMatBd} and using the constant $\Kp$ defined there, the desired result follows.
\end{proof}

The next result provides closed form estimates of the expectation bound given in Theorem~\ref{thm:ExP_Bound_Main_Technical_Result} for the specific stepsize sequence $\alpha_n = 1/(n + 1)^\sigma,$ with $\sigma \in (0,1).$ Notice this family of stepsizes is more general than other common choices in the literature as it is non-square summable for $\sigma \in (0,1/2].$ See Remark~\ref{remark: squre summability} for further details.

\begin{theorem}
	\label{thm:ClosedFormExpectationBoundEst}
	Fix $\sigma \in (0,1)$ and let $\alpha_n = 1/(n + 1)^\sigma.$ Then, under \ref{assum:bounded second moments},
	\begin{align*}
	\ExP\|\theta_{n + 1} - \theta^*\|^2 \leq & \bigg[\Kp e^\lambda  \ExP\|\theta_0 - \theta^*\|^2 e^{- (\lambda/2) (n + 2)^{1 - \sigma}}\\ 
	& +  \frac{2K_s \Kp \Kb e^{\lambda}}{\lambda}\bigg] e^{-(\lambda/2) (n + 2)^{1 - \sigma}}\\
	& + \frac{2K_s \Kp e^{\lambda/2}}{\lambda} \frac{1}{(n + 1)^\sigma},
	\end{align*}
	where $\Kb = e^{[(\lambda/2) \sum_{k = 0}^{i_0} \alpha_k]}$  with $i_0$ denoting a number larger than $(2\sigma/\lambda)^{1/(1 - \sigma)}.$
\end{theorem}

\begin{proof}
	Let $t_n = \sum_{i=0}^{n-1} \alpha_i$ for $n\geq 0$.
	Observe that
	\begin{align*}
		&\sum_{i = 0}^{n} \left[e^{-(\lambda/2) \sum_{k = i + 1}^{n} \alpha_k} \right] \alpha_i \\
		& \leq  \left(\sup_{i \geq 0} e^{(\lambda/2) \alpha_i} \right) \sum_{i = 0}^{n} \left[e^{-(\lambda/2) \sum_{k = i}^{n} \alpha_k} \right] \alpha_i\\
		& =  \left(\sup_{i \geq 0} e^{(\lambda/2) \alpha_i} \right) \sum_{i = 0}^{n} \left[e^{-(\lambda/2) (\tI{n + 1} - \tI{i})} \right] \alpha_i\\
		& \leq  \left(\sup_{i \geq 0} e^{(\lambda/2) \alpha_i} \right) \int_{0}^{\tI{n+ 1}} e^{-(\lambda/2) (\tI{n + 1} - s)} \ds \\
		& \leq  \left(\sup_{i \geq 0} e^{(\lambda/2) \alpha_i} \right) \frac{2}{\lambda}\\
		& \leq  \frac{2e^{\lambda/2}}{\lambda},
	\end{align*}
	where the third relation follows by treating the sum as right Riemann sum, and the last inequality follows since $\sup_{i \geq 0} \alpha_i \leq 1.$ Hence it follows that
	\begin{align}
	&\sum_{i = 0}^{n} \left[e^{-\lambda \sum_{k = i + 1}^{n} \alpha_k} \right] \alpha_i^2
	\\
	& \leq  \left(\sup_{0 \leq i \leq n} \left[\alpha_i e^{-\frac{\lambda}{2} \sum_{k = i + 1}^{n} \alpha_k} \right] \right) \sum_{i = 0}^{n} \left[e^{-\frac{\lambda}{2} \sum_{k = i + 1}^{n} \alpha_k} \right] \alpha_i \nonumber \\
	& \leq \left(\sup_{0 \leq i \leq n} \left[\alpha_i e^{-\frac{\lambda}{2} \sum_{k = i + 1}^{n} \alpha_k} \right] \right) \frac{2e^{\frac{\lambda}{2}}}{\lambda} \label{eqn:Bound 1}. 
	\end{align}
	We claim that for all $n \geq i_0,$
	\begin{equation}
	\label{eqn:SubClaim}
	\sup_{i_0 \leq i \leq n} \left[\alpha_i e^{-(\lambda/2) \sum_{k = i + 1}^{n} \alpha_k} \right] \leq \frac{1}{(n + 1)^\sigma}.
	\end{equation}
	To establish this, we show that for any $n \geq i_0,$ $\alpha_i e^{-(\lambda/2)[ \sum_{k = i + 1}^{n} \alpha_k]}$ monotonically increases as $i$ is varied from $i_0$ to $n.$ To prove the latter, it suffices to show that $\alpha_i e^{-(\lambda/2) \alpha_{i + 1}} \leq \alpha_{i + 1},$ or equivalently $(i + 2)^\sigma/(i + 1)^\sigma \leq e^{\lambda/[2 (i + 2)^\sigma]}$ for all $i \geq i_0.$ But the latter is indeed true. Thus \eqref{eqn:SubClaim} holds. From \eqref{eqn:Bound 1} and \eqref{eqn:SubClaim}, we then have
	\begin{align*}
		\sum_{i = 0}^{n} \Big[&e^{-\lambda \sum_{k = i + 1}^{n} \alpha_k} \Big] \alpha_i^2\\
		\leq
		\tfrac{2e^{\lambda/2}}{\lambda} \bigg[& \left(\sup_{0 \leq i \leq i_0} \left[\alpha_i e^{-(\lambda/2) \sum_{k = i + 1}^{n} \alpha_k} \right] \right) 
		\\
		&+ \left(\sup_{i_0 \leq i \leq n} \left[\alpha_i e^{-(\lambda/2) \sum_{k = i + 1}^{n} \alpha_k} \right] \right) \bigg]\\
		\leq  \tfrac{2e^{\lambda/2}}{\lambda} \bigg[&\left(\sup_{0 \leq i \leq i_0} \left[\alpha_i e^{-(\lambda/2) \sum_{k = i + 1}^{n} \alpha_k} \right] \right) 
		+ \tfrac{1}{(n + 1)^\sigma}\bigg]\\
		\leq  \tfrac{2e^{\lambda/2}}{\lambda} \bigg[& e^{-[(\lambda/2) \sum_{k = 0}^{n} \alpha_k]} \left(\sup_{0 \leq i \leq i_0} \left[\alpha_i e^{(\lambda/2) \sum_{k = 0}^{i} \alpha_k} \right] \right) 
		\\
		&+ \tfrac{1}{(n + 1)^\sigma}\bigg]\\
		\leq  \tfrac{2e^{\lambda/2}}{\lambda} \bigg[& \Kb e^{-[(\lambda/2) \sum_{k = 0}^{n} \alpha_k]} + \tfrac{1}{(n + 1)^\sigma}\bigg],
	\end{align*}
	where the first relation holds as $\sup\{a_0, \ldots, a_n\} \leq \sup\{a_0, \ldots, a_{i_0}\} + \sup\{a_{i_0}, \ldots, a_n\}$ for any positive sequence $\{a_0, \ldots, a_n\}$ with $0 \leq i_0 \leq n,$ and the last relation follows as $\alpha_i \leq 1$ and $\sup_{0 \leq i \leq i_0} e^{(\lambda/2) \sum_{k = 0}^{i} \alpha_k} \leq \Kb.$ Combining the above inequality with the relation from Theorem~\ref{thm:ExP_Bound_Main_Technical_Result}, we have
	\begin{multline*}
	\ExP\|\theta_{n + 1} - \theta^*\|^2 
	\leq \Kp \left[e^{-\lambda \sum\limits_{k = 0}^{n} \alpha_k}\right]\ExP\|\theta_0 - \theta^*\|^2 
	\\+  \frac{2 K_s \Kp e^{\lambda/2}}{\lambda} \left[ \Kb e^{-\left[(\lambda/ 2) \sum\limits_{k = 0}^{n} \alpha_k\right]} + \frac{1}{(n + 1)^\sigma}\right],
	\end{multline*}
	Since
	\[
	\sum_{k = 0}^{n} \alpha_k \geq \int_{0}^{n + 1} \frac{1}{(x + 1)^\sigma} dx = (n + 2)^{1 - \sigma} - 1,
	\]
	the desired result follows.
\end{proof}

To finalize the proof of Theorem~\ref{thm:ExPDecayRate} we employ Theorem~\ref{thm:ClosedFormExpectationBoundEst} with the following constants.
\begin{align}
	K_1 & =  \Kp e^\lambda  \ExP\|\theta_0 - \theta^*\|^2 + \frac{2K_s \Kp \Kb e^{\lambda}}{\lambda}, \label{eq:K1}\\
	K_2 & = \frac{2K_s \Kp e^{\lambda/2}}{\lambda}, \label{eq:K2}
\end{align}
where 
$K_s$ is the constant from \ref{assum:bounded second moments}, 	\[
\Kp :=\max_{\ell_1 \leq \ell_2 \leq m} \prod_{\ell=\ell_1}^{\ell_2} e^{\alpha_\ell (\mu + \lambda)}\enspace
\]
with $\mu = -\lmin(A + A^\tr) + \lmax(A^\tr A + K_s I)$ and $m=\left\lceil\left(\frac{\lmax(A^\tr A + K_s I)}{\lmin(A + A^\tr)- \lambda}\right)^{1/\sigma}\right\rceil,$
and 
\begin{align*}
\Kb &=  \exp{{\left[(\lambda/2) \sum_{k = 0}^{\lceil(2\sigma/\lambda)^{1/(1 - \sigma)}\rceil} \alpha_k\right]}} \\ &\leq \exp{\left[(\lambda/2) \frac{\left(\lceil(2\sigma/\lambda)^{1/(1 - \sigma)}\rceil +1 \right)^{1/\sigma} + \sigma}{1-\sigma}+1 \right]}.
\end{align*}

\section{Proof of Theorem~\ref{thm: convergence rate}}
In this section we prove Theorem~\ref{thm: convergence rate}. Throughout this section we assume \ref{assum:bounded_feat}. All proofs for intermediate lemmas are  given in Appendix~\ref{sec: main thm appendix}. 

\subsection{Outline of Approach} \label{sec:outline}

\begin{table*}[t]
	\begin{center}
		\begin{tabular}{c | c | c | c }
			\hline
			Stepsize  &  Discretization Error & Martingale Noise Impact & TD(0) Behavior\\
			\hline
			& & & \\[-2ex]
			Large & Large & Large & Possibly diverging \\[1ex]
			Moderate & $O(n_0)$  & $O(n_0)$ w.h.p.& Stay in $O(n_0)$ ball w.h.p. \\[1ex]
			Small &  $\epsilon/3$  & $\epsilon/3$ w.h.p. & Converging w.h.p.\\[1ex]
			\hline
		\end{tabular}
	\end{center}
	\caption{\label{tab:AnalysisOutline}Chronological Summary of Analysis Outline}
\end{table*}
The limiting ODE for \eqref{eq:SA_traj} is
\begin{equation}
\label{eq:limiting_ODE}
\dot{\theta}(t) = h(\theta(t)) =  b - A\theta(t) = -A(\theta(t) - \thS) \enspace.
\end{equation}
Let $\theta(t, s, u_0),$ $t \geq s,$ denote the solution to the above ODE starting at $u_0$ at time $t = s.$ When the starting point and time are unimportant, we will denote this solution by $\theta(t)$ .

As the solutions of the ODE are continuous functions of time, we also define a linear interpolation $\{\bart(t)\}$ of $\{\theta_n\}.$ Let $t_0 = 0.$ For $n \geq 0,$ let $\tI{n + 1} = \tI{n} + \alpha_n$ and let
\begin{equation}
\label{eqn:LinInt}
\bart(\tau) \!=\!
\begin{cases}
\theta_n & \! \! \text{ if } \tau = \tI{n} \enspace,\\
\theta_n + \frac{\tau - \tI{n}}{\alpha_n}[\theta_{n + 1} - \theta_n] & \! \! \text{ if } \tau \in (\tI{n}, \tI{n + 1}) \enspace.
\end{cases}
\end{equation}

Our tool for comparing $\bart(t)$ to $\theta(t)$ is the \emph{Variation of Parameters} (VoP) method \cite{lakshmikantham1998method}.
Initially, $\bart(t)$ could stray away from $\thS$ when the stepsizes may not be small enough to tame the noise. However,  we show that $\|\bart(\tI{n}) - \thS\| = O(n),$ i.e., $\theta_n$ does not stray away from $\thS$ too fast.  Later, we show that we can fix some $n_0$ so that first the TD(0) iterates for $n \geq n_0$ stay within an $O(n_0)$ distance from $\thS.$ Then, after for some additional time, when the stepsizes decay enough, the TD(0) iterates start behaving almost like a noiseless version. These three different behaviours are summarized in Table~\ref{tab:AnalysisOutline} and illustrated in Figure~\ref{fig:trajectory}.

\begin{figure*}
	\begin{center}
	\includegraphics[scale=0.25]{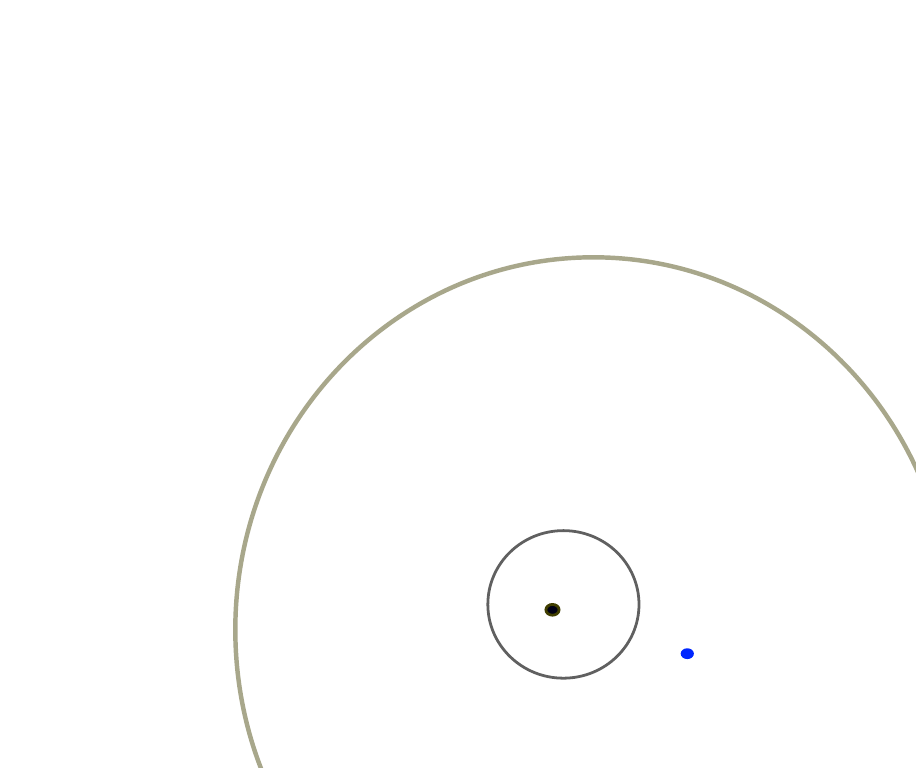}
	\end{center}
	\caption{Visualization of the proof outline. The three balls (from large to small) are respectively the $2\Ro(n_0)$ ball, $\Ro(n_0)$ ball, and $\ei$ ball, where $\Ro(n_0)$ is from Lemma~\ref{lem:WorstCaseThetaBd}.
		The blue curve is the initial, possibly diverging phase of $\bart(t)$. The green curve is $\bart(t)$ when the stepsizes are moderate in size ($t_{n_0} \leq t \leq t_{\nMid}$ in the analysis). Similarly, the red curve is $\bart(t)$ when the stepsizes are sufficiently small ($ t > t_{\nMid}$). The dotted curves are the associated ODE trajectories $\theta(t,t_n,\theta_n)$.}
	\label{fig:trajectory}
\end{figure*}

\subsection{Preliminaries}
We establish some preliminary results here that will be used throughout this section.
Let $s \in \Real,$ and $u_0 \in \dReal.$ Using results from Chapter 6, \cite{hirsch2012differential}, it follows that the solution $\theta(t, s, u_0),$ $t \geq s,$ of \eqref{eq:limiting_ODE} satisfies the relation
\begin{equation} \label{eq:ODE_traj}
\theta(t, s, u_0) = \thS + e^{-A(t - s)} (u_0 - \thS) \enspace.
\end{equation}
As the matrix $A$ is positive definite, for $\theta(t) \equiv \theta(t, s, u_0),$
\[
\frac{d}{dt}\|\theta(t) - \thS\|^2 = -2(\theta(t) - \thS)^\top A (\theta(t) - \thS)<0 \enspace.
\]
Hence
\begin{equation}
\label{eq: norm of thetan - thetastar mon dec}
\|\theta(t', s ,u_0) - \thS\| \leq \|\theta(t, s,u_0) - \thS\|
\enspace,
\end{equation}
for all $t' \geq t \geq s$ and $u_0.$

Let $\lambda$ be as in Theorem~\ref{thm: convergence rate}. From Corollary 3.6, p71, \cite{teschl2012ordinary},  $\exists \Kl \geq 1$ so that $\forall t \geq s$
\begin{equation}
\label{eq:expMatBd}
\|e^{-A(t - s)}\| \leq \Kl e^{-\lambda (t - s)} \enspace.
\end{equation}
Separately, as  $t_{n+1}-t_{k+1} = \sum_{\ell = k + 1}^{n} \alpha_\ell = \sum_{\ell=k+1}^n \tfrac{1}{\ell+1},$
\begin{equation}
\label{eq:bounding the exp of tk}
\frac{(k + 1)^\lambda}{(n + 1)^{\lambda}} \leq e^{-\lambda(\tI{n + 1} - \tI{k + 1})} \leq \frac{(k + 2)^\lambda}{(n + 2)^\lambda} \enspace.
\end{equation}

The following result is a consequence of \ref{assum:bounded_feat} that gives a bound directly on the martingale difference noise as a function of the iterates. We emphasize that this strong behavior of TD(0) is significant in our work. We also are not aware of other works that utilized it even though \ref{assum:bounded_feat} or equivalents are often assumed and accepted. 

\begin{lemma}[Martingale Noise Behavior] \label{lem:martingale_bound_TD0}
	For all $n \geq 0,$
	\[
	\|M_{n + 1}\|\leq \Km [1 + \|\theta_n - \thS\|] \enspace ,
	\]
	where
	\[
	\Km := \frac{1}{4}\max \left\{2 + [1 + \gamma] \|A^{-1}\| \|b\| , 1 + \gamma + 4 \| A \| \right\} \enspace.
	\]
\end{lemma}

\begin{remark} \label{rem: weak noise}
	The noise behavior usually used in the literature (e.g., \cite{sutton2009fast,sutton2009convergent}) is the same as we assumed in \ref{assum:bounded second moments} for Theorem~\ref{thm:ExPDecayRate}:
	\[
	\bE[||M_{n+1}||^2|{\cal F}_n]\leq K_s(1+||\theta_n||^2)\enspace,
	\]
	for some constant $K_s \geq 0$. However, here we assume the stronger \ref{assum:bounded_feat}, which, using a similar proof technique to that of Lemma~\ref{lem:martingale_bound_TD0}, implies
	\[
	\||M_{n + 1}||^2\leq 3[1+\gamma + \max(\|A\|,\|b\|)]^2(1 + ||\theta_n||^2)
	\]
	for all $n \geq 0.$	
\end{remark}

The remaining parts of the analysis rely on the comparison of the discrete TD(0) trajectory $\{\theta_n\}$ to the continuous solution $\theta(t)$ of the limiting ODE. For this, we first switch from directly treating $\{\theta_n\}$ to treating their linear interpolation $\{\bart(t)\}$ as defined in \eqref{eqn:LinInt}. The key idea then is to use the VoP method \cite{lakshmikantham1998method} as in Lemma~\ref{lem:vopApp}, and express $\bart(t)$ as a perturbation of $\theta(t)$ due to two factors: the discretization error and the martingale difference noise. 
Our quantification of these two factors is as follows. 
For the interval $[\tI{\ell_1}, \tI{\ell_2}],$ let
$$
E^\text{d}_{[\ell_1,\ell_2]} := \sum_{k = \ell_1}^{\ell_2-1}\int_{\tI{k}}^{\tI{k + 1}} e^{-A (\tI{n + 1} - \tau) } A [\bart(\tau) - \theta_k] \df \tau \enspace,
$$
and
$$
E^\text{m}_{[\ell_1,\ell_2]} := \sum_{k = \ell_1}^{\ell_2-1}\left[\int_{\tI{k}}^{\tI{k + 1}} e^{-A (\tI{n + 1}  - \tau)}\df \tau\right] M_{k + 1} \enspace.
$$

\begin{corollary}[Comparison of SA Trajectory and ODE Solution] \label{cor:ODE perturbation}
	For every $\ell_2 \geq \ell_1$,
	\[
	\bart(\tI{\ell_2}) - \thS  =  \theta(\tI{\ell_2}, \tI{\ell_1}, \bart(\tI{\ell_1})) - \thS + E^\text{d}_{[\ell_1,\ell_2]} + E^\text{m}_{[\ell_1,\ell_2]} \enspace.
	\]
\end{corollary}
We highlight that both the paths, $\bart(t)$ and $\theta(t, \tI{\ell_1}, \bart(\tI{\ell_1})),$ $t \geq \tI{\ell_1},$ start at the same point $\bart(\tI{\ell_1})$ at time $\tI{\ell_1}.$ 
Consequently, by bounding $E^\text{d}_{[\ell_1,\ell_2]}$ and $E^\text{m}_{[\ell_1,\ell_2]}$ we can estimate the distance of interest.

\subsection{Part I -- Initial Possible Divergence}
\label{sec:phase1}

In this section, we show that the TD(0) iterates lie in an $O(n)$-ball around $\thS.$ We stress that this is one of the results that enable us to accomplish more than existing literature. Previously, the distance of the initial iterates from $\thS$ was bounded using various assumptions,  often justified  with an artificial projection step which we are able to avoid.

Let $R_0 := 1  + \|\theta_0 - \thS\|.$
\begin{lemma}[Worst-case Iterates Bound]
	\label{lem:WorstCaseThetaBd}
	For $n \geq 0,$
	\[
	\|\theta_n - \thS\| \leq \Ro(n) 
	\enspace,
	\]
	where
	\[
	\Ro(n) := [n+1] \cS R_0
	\]
	and
	$\cS:= 1 + \|\thS\| \leq 1 + \|A^{-1}\| \; \|b\|$
\end{lemma}

Next, since $\|M_{n+1}\|$ is linearly bounded by $\|\theta_n - \thS\|$, the following result shows that $\|M_{n+1}\|$ is $O(n)$ as well. It  follows from Lemmas~\ref{lem:martingale_bound_TD0} and \ref{lem:WorstCaseThetaBd}.

\begin{corollary}[Worst-case Noise Bound]
	\label{cor:WorstCaseMBd}
	For $n \geq 0,$
	\[
	\|M_{n + 1}\| \leq  \Km [1 + \cS R_0] [n + 1] \enspace.
	\]
\end{corollary}

\subsection{Part II -- Rate of Convergence}

%
%
%
%

Here, we bound the probability of the event
\begin{equation*}
\cE(n_0, n_1) := \{\|\theta_n - \thS\| \leq \epsilon \; \forall n > n_0 + n_1\} \enspace
\end{equation*}
for sufficiently large $n_0, n_1;$ how large they should be will be elaborated later.
We do this by comparing the TD(0) trajectory $\theta_{n}$ with the ODE solution $\theta(\tI{n}, \tI{n_0}, \bart(\tI{n_0}))$ $\forall n \geq n_0$; for this we will use Corollary ~\ref{cor:ODE perturbation} along with Lemma~\ref{lem:WorstCaseThetaBd}.  Next, we show that if $n_0$ is sufficiently large, or equivalently the stepsizes $\{\alpha_{n}\}_{n \geq n_0}$ are small enough, then after waiting for a finite number of iterations from $n_0,$ the TD(0) iterates are $\epsilon-$close to $\thS$ w.h.p. The sufficiently long waiting time ensures that the ODE solution $\theta(t_{n + 1}, t_{n_0}, \bar{\theta}_{n_0})$ is $\epsilon-$close to $\thS;$ the small stepsizes ensure that the discretization error and martingale difference noise are small enough.

Let $\delta \in (0, 1),$ and let $\epsilon$ be such that $\epsilon>0.$ Also, for an event $\cE,$ let $\cE^c$ denote its complement and let $\{\cE_1, \cE_2\}$ denote $\cE_1 \cap \cE_2.$ We begin with a careful decomposition of $\cE^c(n_0, n_1),$ the complement of the event of interest. The idea is to break it down into an incremental union of events. Each such event has an inductive structure: good up to iterate $n$ (denoted by $\Gnd$ below) and the $(n + 1)-$th iterate is bad. The good event $\Gnd$ holds when all the iterates up to $n$ remain in an $O(n_0)$ ball around $\thS.$ For $n < n_0 + n_1,$ the bad event means that $\theta_{n + 1}$ is outside the $O(n_0)$ ball around $\thS,$ while for $n \geq n_0 + n_1,$ the bad event means that $\theta_{n + 1}$ is outside the $\ei$ ball around $\thS.$  Formally, for $n_1 \geq 1,$ define the events
\[
\Emid := \hspace{-0.5em}\bigcup_{n = n_0} ^{n_0 + n_1 - 1} \! \left\{ \Gnd, \|\theta_{n + 1} - \thS\| \!  > \!2 \Ro(n_0) \right\} \enspace,
\]
\begin{align}
&\Eaft 
\\
&:= \bigcup_{n = n_0 + n_1} ^{\infty} \left\{\Gnd,  \|\theta_{n + 1} - \thS\| > \min\{\epsilon, 2\Ro(n_0) \} \right\} \;,
\end{align}
and, $\forall n \geq n_0,$ let
\[
\Gnd \! := \! \left\{\bigcap_{k = n_0}^{n} \!\{\|\theta_k - \thS\| \! \leq \! 2\Ro(n_0) \}\right\} \enspace.
\]

Using the above definitions, the decomposition of $\cE^c(n_0, n_1)$ is the following relation.
\begin{lemma}[Decomposition of Event of Interest] \label{lem: decomposition}
	For $n_0, n_1 \geq 1,$
	\[
	\cE^c(n_0, n_1) \subseteq \Emid \cup \Eaft \enspace.
	\]
\end{lemma}

For the following results, define the constants
%
%
\[
\cMb :=
\begin{cases}
\frac{6 \Km \Kl 2^{\lambda - 0.5}}{\sqrt{2 \lambda - 1}} & \text{ if $\lambda > 0.5$}\\
\frac{6 \Km \Kl }{\sqrt{1 - 2\lambda}} & \text{ if $\lambda < 0.5$ \enspace.}%
\end{cases}
\]
%

Next, we show that on the ``good'' event $\Gnd,$ the discretization error is small for all sufficiently large $n.$


\begin{lemma}[Part II Discretization Error Bound]
	\label{lem:SmallDE}
	For any
	$$
	n \geq n_0 \geq \tfrac{K_\lambda6  \|A\|(\|A\| + 2\Km)}{\lambda},
	$$
	\[
	\|E_{[n_0,n+1]}^d\| \leq \tfrac{1}{3}[n_0+1]C_*R_0 = \tfrac{1}{3} \Ro(n_0).
	\]
	
	Furthermore, for
	$$n \geq \nMid \geq \left( 1 + \tfrac{K_\lambda6  \|A\| (\|A\| + 2\Km) C_* R_0}{\lambda \min\{ \epsilon, \Ro(n_0)\}} \right)(n_0 + 1)$$
	it thus also holds on $G_{n_0,n}$ that
	\begin{align*}
	\|E_{[\nMid,n+1]}^d\| &\leq \tfrac{1}{3}\min\{\epsilon, [n_0+1]C_*R_0\} \\ &= \tfrac{1}{3}\min\{\epsilon, \Ro(n_0)\}\enspace.
	\end{align*}
\end{lemma}

The next result gives a bound on the probability that, on the ``good'' event $\Gnd,$ the martingale difference noise is small when $n$ is large. The bound has two forms for the different values of $\lambda$.

\begin{lemma}[Part II Martingale Difference Noise Concentration]
	\label{lem:MartConc}
	Let $n_0 \geq 1$ and $R \geq 0.$ Let $n \geq n' \geq n_0.$
	\begin{itemize}
		\item For $\lambda > 1/2,$
		\begin{align*}
		\Pr\{\Gnd, &\|E_{[n',n+1]}^{m}\| \geq R \} \\
		&\leq 2d^2 \exp\left[- \frac{(n + 1) R^2 }{ 2d^3 \cMb^2 \Ro^2(n_0)}\right] \enspace.
		\end{align*}
		
		\item For $\lambda < 1/2,$
		\begin{align*}
		\Pr\{\Gnd, &\|E_{[n',n+1]}^{m}\| \geq R \} \\ &\leq 2d^2 \exp\left[-\frac{[n' + 1]^{1 - 2 \lambda} (n + 1)^{2 \lambda} R^2}{2d^3 \cMb^2 \Ro^2(n_0)}\right] \enspace.
		\end{align*}
	\end{itemize}
	
\end{lemma}

Having Lemma~\ref{lem:SmallDE}, we substitute $R = \tfrac{\Ro(n_0)}{2}$ in Lemma~\ref{lem:MartConc} and estimate the resulting sum to bound $\Emid$.

\begin{lemma}[Bound on Probability of $\Emid$] \label{lem: bound on Emid}
	Let $n_0 \geq \max\left\{\tfrac{K_\lambda6  \|A\| (\|A\| + 2\Km)}{\lambda}, 2^{\frac{1}{\lambda}}\right\}$ and $n_1 \geq 1.$
	\begin{itemize}
		\item For $\lambda > 1/2,$ 
		\[
		\Pr\{\Emid \} \leq  16d^5 \cMb^2  \exp\left[-\frac{n_0}{8d^3 \cMb^2 }\right] \enspace.
		\]
		
		\item For $\lambda < 1/2,$ 
		\begin{equation*}
		\Pr\{\Emid \} \leq \\ 2d^2 \left[\frac{8d^3 \cMb^2}{ \lambda}\right]^{\frac{1}{2\lambda}}  \frac{\exp[- \frac{n_0}{64d^3 \cMb^2}]}{(n_0 + 1)^{\frac{1 - 2 \lambda}{2 \lambda}}} \enspace.
		\end{equation*}
		
	\end{itemize}
\end{lemma}

Lastly, we upper bound $\Eaft$ in the same spirit as $\Emid$ in Lemma~\ref{lem: bound on Emid}, again using Lemmas~\ref{lem:SmallDE} and \ref{lem:MartConc}; this time with $ R = \frac{\ei}{3}$ . 
\begin{lemma}[Bound on Probability of $\Eaft$] \label{lem: bound on Eaft}
	Let $$n_0 \geq \max\left\{\tfrac{K_\lambda6  \|A\| (\|A\| + 2\Km)}{\lambda}, 2^{\frac{1}{\lambda}}\right\}$$ and
	\[
	\nMid \geq \left(1  + \tfrac{K_\lambda6 \|A\| (\|A\| + 2\Km) }{\lambda \min\{ \epsilon, \Ro(n_0)\}} \right)\Ro(n_0).
	\]
	%
	Let $n_1 \equiv n_1(\ei,\nMid,n_0) \geq (\nMid + 1) \left[ \frac{6\Kl \Ro(n_0)}{\ei} \right]^{1/\lambda} - n_0.$
	\begin{itemize}
		\item
		For $\lambda > 1/2,$
		\begin{align*}
		\Pr\{&\Eaft \}
		\leq  36 d^5 \cMb^2 \left[\frac{\Ro(n_0)}{\ei}\right]^2\\
		&\times \exp \left[- \frac{(6\Kl)^{1/\lambda}}{ 18 d^3 \cMb^2} (\nMid + 1) \left[\frac{\ei}{\Ro(n_0)}\right]^{2 - \tfrac{1}{\lambda}}  \right] .
		\end{align*}
		\item
		For $\lambda < 1/2,$ 
		\begin{align*}
		\Pr\{\Eaft\}
		\leq
		2d^2& \left[ \frac{ 18 d^3 \cMb^2  [\Ro(n_0)]^2}{\ei^2 \lambda } \right]^{\frac{1}{2\lambda}} \\
		&\times \exp\left[-\frac{K_\lambda^2 }{4 d^3 \cMb^2 } (\nMid + 1) \right].
		\end{align*}
	\end{itemize}
\end{lemma}

We are now ready to put the pieces together for proving Theorem~\ref{thm: convergence rate}. For the detailed calculations see end of Appendix~\ref{sec: main thm appendix}.
\begin{proof}[Proof of Theorem~\ref{thm: convergence rate}]
From Lemma~\ref{lem: decomposition}, by a union bound,
\begin{align*}
\Pr\{\cE^c(n_0, n_1)\}
\leq
\Pr\{\Emid \} + \Pr\{\Eaft\} \enspace.
\end{align*}
The behavior of $\Emid$ is dictated by $n_0$, while the behavior of $\Eaft$ by $n_1$.  Using Lemma~\ref{lem: bound on Emid}, we set $n_0$ so that $\Emid$ is less than $\delta/2$, resulting in the condition $n_0  = O\left(\ln\tfrac{1}{\delta}\right)$. Next, using Lemma~\ref{lem: bound on Eaft}, we set $n_1$ so that $\Eaft$ is less than $\delta/2$, resulting in 
$$
n_1 = \tilde{O}\left(\big[{(1/\epsilon)}\ln{(1/\delta)}\big]^{\max\left\{1+{1/\lambda},2\right\}}\right)
$$ for 
$
\lambda > 1/2,
$
and 
$$
n_1 =\tilde{O}\left(\big[{(1/\epsilon)}\ln{(1/\delta)}\big]^{1+{1/\lambda}}\right)
$$ for 
$\lambda < 1/2.
$
\end{proof}

\section{Discussion}
In this work, we obtained the first convergence rate estimates for an unaltered version of the celebrated TD(0). It is, in fact, the first to show rates of an unaltered online TD algorithm of any type. 

As can be seen from Theorem~\ref{thm: convergence rate}, the bound explodes when the matrix $A$ is ill-conditioned. We stress that this is not an artifact of the bound but an inherent property of the algorithm itself. This happens because along the eigenspace corresponding to the zero eigenvalues, the limiting ODE makes no progress and consequently no guarantees for the (noisy) TD(0) method can be given in this eigenspace. As is well known, the ODE will, however, advance in the eigenspace corresponding to the non-zero eigenvalues to a solution which we refer to as the truncated solution. Given this, one might expect that the (noisy) TD(0) method may also converge to this truncated solution. We now provide a short example that suggests that this is in fact not the case. Let 
\[
A := \begin{bmatrix}
1 & 1 \\
0 & 0
\end{bmatrix},
\text{ and } 
b := \begin{bmatrix}
2\\
0
\end{bmatrix}.
\]
Clearly, $\thS := \begin{bmatrix} 1 & 1 \end{bmatrix}^\tr$ is a vector satisfying $b = A \thS$ and the eigenvalues of $A$ are $1$ and $0.$ Consider the update rule $\theta_{n + 1} = \theta_n + \alpha_n[b -A \theta_n + M_{n + 1}]$ with
\[
M_{n + 1} = \begin{bmatrix} 1 \\ 1 \end{bmatrix} Z_{n + 1} [\theta_n(2) - \thS(2)].
\]
Here $v(i)$ is the $i-$th coordinate of vector $v,$ and $\{Z_n\}$ are IID Bernoulli $\{-1, +1\}$ random variables. For an initial value $\theta_0,$ one can see that the (unperturbed) ODE for the above update rule converges to $\begin{bmatrix} -1 &  1 \end{bmatrix}^\tr \theta_0(2) + b;$  this is not $\thS,$ but the truncated solution mentioned above. For the same initial point, predicting the behavior of the noisy update is not easy. Rudimentary simulations show the following. In the initial phase (when the stepsizes are large) the noise dictates how the iterates behave. Afterwards, at a certain stage when the stepsizes become sufficiently small, an “effective $\theta_0$” is detected, from which the iterates start converging to a new truncated solution, corresponding to this “effective $\theta_0$”. This new truncated solution is different per each run and is often very different from the truncated solution corresponding the initial iterate $\theta_0.$

Separately, we stress that our proof technique is general and can be used to provide convergence rates for TD with non-linear function approximations, such as neural networks. Specifically, this can be done using the non-linear analysis presented in \cite{thoppe2015concentration}. There, the more general form of Variation of Parameters is used: the so-called Alekseev's formula. However, as mentioned in Section~\ref{sec:intro}, the caveat there is that the $n_0-$th iterate needs to be in the domain of attraction of the desired asymptotically stable equilibrium point. 
Nonetheless, we believe that one should be able to extend our present approach to non-linear ODEs with a unique global equilibrium point. For non-linear ODEs with multiple stable points, the following approach can be considered. In the initial phase, the location of the SA iterates is a Markov chain with the state space being the domain of attraction associated with different attractors \cite{williams2002stability}. Once the stepsizes are sufficiently small, analysis as in our current paper via Alekseev's formula may enable one to obtain expectation and high probability convergence rate estimates. In a similar fashion, one may obtain such estimates even for the two timescale setup by combining the ideas here with the analysis provided in \cite{dalal2017concentration}.

Finally, future work can extend to a more general family learning rates, including the commonly used adaptive ones. Building upon Remark~\ref{rem: weak noise}, we believe that a stronger expectation bound may hold for TD(0) with uniformly bounded features and rewards. This may enable obtaining tighter convergence rate estimates for TD(0) even with generic stepsizes. 

\section{Acknowledgments}
This research was supported by the European Community's Seventh Framework Programme (FP7/2007-2013) under grant agreement 306638 (SUPREL). A portion of this work was completed when Balazs Szorenyi and Gugan Thoppe were postdocs at Technion, Israel. Gugan's research was initially supported by ERC grant 320422 and is now supported by grants NSF IIS-1546331, NSF DMS-1418261, and NSF DMS-1613261.

\bibliography{TD0_References}
\bibliographystyle{aaai}

\newpage

\appendix
\onecolumn

\section{Variation of Parameters Formula}
\label{sec:VoP}

Let $\theta(t, s, \bart(s)),$ $t\geq s,$ be the solution to \eqref{eq:limiting_ODE} starting at $\bart(s)$ at time $t = s.$ For $k \geq 0,$ and $\tau \in [\tI{k}, \tI{k + 1}),$ let
\begin{equation}
\label{Defn:DiscEr}
\zeta_1(\tau) := h(\theta_k) - h(\bart(\tau)) = A[\bart(\tau) - \theta_k]
\end{equation}
and
\begin{equation}
\label{Defn:MartEr}
\zeta_2(\tau) := M_{k + 1} \enspace .
\end{equation}

\begin{lemma}
	\label{lem:vopApp}
	Let $i \geq 0.$ For $t \geq \tI{i}.$
	\[
	\bart(t) = \theta(t, \tI{i}, \bart(\tI{i})) + \int_{\tI{i}}^{t}e^{-A(t - \tau)} [\zeta_1(\tau) + \zeta_2(\tau)]\df \tau.
	\]
\end{lemma}
\begin{proof}
	For $n \geq 0$ and $t \in [\tI{n}, \tI{n + 1}),$ by simple algebra,
	\[
	\bart(t) - \bart(\tI{i}) = \frac{t - \tI{n}}{\alpha_n} [\theta_{n + 1} - \theta_n] + \sum_{k = i}^{n - 1}[\theta_{k + 1} - \theta_k].
	\]
	Combining this with \eqref{eq:SA_traj}, \eqref{Defn:DiscEr}, and \eqref{Defn:MartEr}, and using the relations $\tau - \tI{n} = \int_{\tI{n}}^{t} \df \tau$ and $\alpha_k = \int_{\tI{k}}^{\tI{k + 1}} \df \tau,$  we have
	\[
	\bart(t) = \bart(\tI{i}) + \int_{\tI{i}}^{t} h(\bart(\tau)) \df \tau + \int_{\tI{i}}^{t} [\zeta_1(\tau) + \zeta_2(\tau)] \df \tau.
	\]
	Separately, writing \eqref{eq:limiting_ODE} in integral form, we have
	\[
	\theta(t, \tI{i}, \bart(\tI{i})) = \bart(\tI{i}) + \int_{\tI{i}}^{t} h(\theta(\tau)) \df \tau.
	\]
	From the above two relations and the \vop formula \cite{lakshmikantham1998method}, the desired result follows.
\end{proof}

\section{Supplementary Material for Proof of Theorem~\ref{thm: convergence rate}}
\label{sec: main thm appendix}
\begin{proof}[Proof of Lemma \ref{lem:martingale_bound_TD0}]
	We have
	\begin{eqnarray*}
		\|M_{n+1}\| & = & \| r_n\phi_n + (\gamma \phi_n' - \phi_n)^\top\theta_n\phi_n - [b - A\theta_n] \|\\
		& = & \| r_n\phi_n + (\gamma \phi_n' - \phi_n)^\top(\theta_n-\thS)\phi_n  \\
		& & \quad+(\gamma\phi_n'-\phi_n)^\top\thS\phi_n + A(\theta_n-\thS) \| \\
		& \leq  &\frac{1}{2} + \frac{[1 + \gamma]}{4} \| A^{-1} \| \; \| b \|  + \frac{[1 + \gamma + 4 \| A \|]}{4} \|\theta_n-\thS\|,
	\end{eqnarray*}
	where the first relation follows from \eqref{Defn:MartDiffNoise}, the second holds as $b = A \thS,$ while the third follows since \ref{assum:bounded_feat} holds and $\thS = A^{-1} b.$ The desired result is now easy to see. \end{proof}

\begin{proof}[Proof of Corollary~\ref{cor:ODE perturbation}]
	The result follows by using Lemma~\ref{lem:vopApp} from Appendix~\ref{sec:VoP}, with $i = \ell_1,$ $t = \tI{\ell_2},$ and subtracting $\thS$ from both sides.
\end{proof}

\begin{proof}[Proof of Lemma~\ref{lem:WorstCaseThetaBd}]
	The proof is by induction. The claim holds trivially for $n = 0.$ Assume the claim for $n.$ Then from \eqref{eqn:TD_Method},
	\begin{equation*}
	\|\theta_{n + 1} - \thS\|  \leq \|\theta_n - \thS\| + \alpha_n\| [\gamma \phi_n' - \phi_n]^\tr \thS \phi_n\| 
	+ \alpha_n \|r_n \phi_n\| + \alpha_n\| [\gamma \phi_n' - \phi_n]^\tr[\theta_n - \thS] \phi_n\| \enspace.
	\end{equation*}
	Applying the Cauchy-Schwarz inequality, and using \ref{assum:bounded_feat} and the fact that $\gamma \leq 1,$ we have
	\[
	\|\theta_{n + 1} - \thS\| \leq \|\theta_n - \thS\| + \frac{ \alpha_n}{2} \cS + \frac{\alpha_n}{2} \|\theta_n - \thS\|.
	\]
	Now as $1 \leq R_0,$ we have
	\[
	\|\theta_{n + 1} - \thS\| \leq \left[1 + \frac{\alpha_n}{2} \right]\|\theta_n - \thS\| + \frac{ \alpha_n}{2} \cS R_0.
	\]
	Using the induction hypothesis and the stepsize choice, the claim for $n + 1$ is now easy to see. The desired result thus follows.
\end{proof}

\begin{proof}[Proof of Lemma~\ref{lem: decomposition}]
	For any two events $\cE_1$ and $\cE_2,$ note that
	\begin{equation}
	\label{eqn:sDec1}
	\cE_1 = [\cE_2^c  \cap \cE_1] \cup [\cE_2 \cap \cE_1] \subseteq \cE_2^c \cup [\cE_2 \cap \cE_1] \enspace .
	\end{equation}
	Separately, for any sequence of events $\{\cE_k\},$ observe that
	\begin{equation}
	\label{eqn:sDec2}
	\bigcup_{k = 1}^{m} \cE_k =  \left[\bigcup_{k = 1}^{m} \left(\left[\bigcup_{i = 1}^{k - 1} \cE_i\right]^c  \cap \cE_k\right)\right] \enspace,
	\end{equation}
	where $\bigcup_{i = i_1}^{i_2} \cE_i  = \emptyset$ whenever $i_1 > i_2.$ Using \eqref{eqn:sDec1}, we have
	\begin{equation}
	\label{eqn:Split}
	\cE^c(n_0, n_1) \subseteq G_{n_0, n_0 + n_1}^c
	\cup [G_{n_0, n_0 + n_1} \cap \cE^c(n_0, n_1)] \enspace.
	\end{equation}
	From Lemma~\ref{lem:WorstCaseThetaBd}, $\{\|\theta_{n_0}- \thS\| \leq \Ro(n_0)\}$ is a certain event. Hence it follows from \eqref{eqn:sDec2} that
	\begin{equation}
	\label{eqn:DT1}
	G_{n_0, n_0 + n_1}^c = \Emid.
	\end{equation}
	Similarly, from \eqref{eqn:sDec2} and the fact that $\epsilon \leq R_0,$
	\begin{equation}
	\label{eqn:DT2}
	G_{n_0, n_0 + n_1} \cap \cE^c(n_0, n_1) \subseteq \Eaft \enspace.
	\end{equation}
	Substituting \eqref{eqn:DT1} and \eqref{eqn:DT2} in \eqref{eqn:Split} gives
	\[
	\cE^c(n_0, n_1) \subseteq \Emid \cup \Eaft \enspace.
	\]
	The claimed result follows.
\end{proof}

\begin{proof}[Proof of Lemma~\ref{lem:SmallDE}]
	For $n \geq n' \geq n_0 \geq 0$,
	by its definition and the triangle inequality,
	\[
	\| \eDbp\| \! \leq \! \sum_{k = n'}^{n}\! \! \int_{\tI{k}}^{\tI{k + 1}} \hspace{-1em} \|e^{-A(\tI{n + 1} \! - \tau)}\| \| A\| \| \bart(\tau) - \theta_k \| \df \tau.
	\]
	Fix a $k \in \{n', \ldots, n\}$ and $\tau \in [\tI{k}, \tI{k + 1}).$ Then using \eqref{eqn:LinInt}, \eqref{eq:SA_traj}, \eqref{Defn:h}, and the fact that $(\tau - \tI{k}) \leq \alpha_k,$ we have
	\[
	\|\bart(\tau) - \theta_k\| \leq \alpha_k[\|A\| \|\theta_k - \thS\| + \|M_{k + 1}\|] \enspace.
	\]
	Combining this with Lemma~\ref{lem:martingale_bound_TD0}, we get
	\[
	\|\bart(\tau) - \theta_k\| \leq \alpha_k[\Km + (\|A\| + \Km) \|\theta_k - \thS\|] \enspace.
	\]
	As the event $\Gnd$ holds, and since
	$\alpha_k \leq \alpha_{n'}$ and $\Ro(n_0) \geq 1,$
	we have
	\[
	\|\bart(\tau) - \theta_k\| \leq 2[\|A\| + 2\Km] \alpha_{n'}[n_0+1] C_* R_0 \enspace.
	\]
	From the above discussion, \eqref{eq:expMatBd}, the stepsize choice, and the facts that
	\[
	\sum_{k = n'}^{n} \int_{\tI{k}}^{\tI{k + 1}} \hspace{-0.5em} e^{-\lambda(\tI{n + 1} - \tau)} \df \tau
	= \int_{\tI{n'}}^{\tI{n + 1}} \hspace{-0.5em}  e^{-\lambda(\tI{n + 1} - \tau)} \df \tau
	\leq \frac{1}{\lambda} \enspace,
	\]
	and $\alpha_k \leq \alpha_{n'} \leq \alpha_{n_0},$ we get

	\[
	\|\eDbp\|
	\leq
	\tfrac{K_\lambda 2 \|A\| (\|A\| + 2\Km)(n_0+1) C_* R_0}{\lambda (n'+1)}
	\enspace.
	\]
	The desired results now follow by substituting $n'$ first with $n_0$ and then with $n_c.$
\end{proof}

\begin{proof}[Proof of Lemma~\ref{lem:MartConc}]
	Let $Q_{k, n} = \int_{\tI{k}}^{\tI{k + 1}} e^{-A(\tI{n + 1} - \tau)} \df \tau.$ Then, for any $n_0 \leq n' \leq n,$
	\[
	E_{[n',n+1]}^{m} = \sum_{k = n'}^{n} Q_{k, n} M_{k + 1} \enspace,
	\]
	a sum of martingale differences. When the event $\Gnd$ holds, it follows that the indicator $1_{G_{n_0,k}} = 1$ $\forall k \in \{n_0, \ldots, n', \ldots n\}.$ Hence, for any $R \geq 0,$
	\begin{eqnarray*}
		\Pr\{\Gnd, \|E_{[n',n+1]}^{m}\| \geq R \} & = & \Pr\left\{\Gnd, \left\|\sum_{k = n'}^{n} Q_{k, n} M_{k + 1} 1_{G_{n_0,k}}\right\| \geq R  \right\} \\
		& \leq & \Pr\left\{\left\|\sum_{k = n'}^{n} Q_{k, n} M_{k + 1} 1_{G_{n_0,k}}\right\| \geq R  \right\}. %
	\end{eqnarray*}
	
	Let $Q^{ij}_{k, n}$ be the $i,j-$th entry of the matrix $Q_{k, n}$ and let $M_{k + 1}^j$ be the $j-$th coordinate of $M_{k + 1}.$ Then using the union bound twice on the above relation, we have
	\[
	\Pr\{\Gnd, \|E_{[n',n+1]}^{m}\| \geq R \} \leq \sum_{i= 1}^{d} \sum_{j = 1}^d\Pr\left\{\left|\sum_{k = n'}^{n} Q^{ij}_{k, n} M^j_{k + 1} 1_{G_{n_0,k}}\right| \geq \frac{R}{d \sqrt{d}}  \right\}.
	\]
	As $|Q_{k, n}^{ij} M_{k + 1}^j |  1_{G_{n_0,k}} \leq \|Q_{k, n}\| \|M_{k + 1}\|  1_{G_{n_0,k}} =: \beta_{k, n},$ Azuma-Hoeffding inequality now gives
	\begin{equation}
	\Pr\{\Gnd, \|E_{[n',n+1]}^{m}\| \geq R \} \leq 2d^2 \exp\left[-\frac{R^2}{2 d^3 \sum_{k =n'}^{n} \beta_{k, n}^2}\right] \enspace. \label{eqn:IntConcBd}
	\end{equation}
	On the event $G_{n_0, k},$  $\|\theta_{k} - \thS\| \leq 2 \Ro(n_0)$ by definition. Hence from Lemma~\ref{lem:martingale_bound_TD0}, we have
	\begin{equation}
	\label{eqn:CondMBd}
	\|M_{k + 1}\| 1_{G_{k}} \leq 3\Km \Ro(n_0) \enspace.
	\end{equation}
	Also from \eqref{eq:expMatBd}, $\|Q_{k, n} \| \leq \Kl e^{- \lambda (\tI{n + 1} - \tI{k + 1})} \alpha_k.$ Combining the two inequalities, and using \eqref{eq:bounding the exp of tk} along with the fact that $1/(k + 1) \leq 2/(k + 2),$ we get
	\begin{eqnarray*}
		\beta_{k, n} & \leq &  3 \Km \Kl \Ro(n_0)  e^{-\lambda(\tI{n + 1} - \tI{k + 1})} \alpha_k\\
		& \leq & 6 \Km \Kl  \Ro(n_0) \frac{(k + 2)^{\lambda - 1}}{(n + 2)^\lambda } \enspace.
	\end{eqnarray*}

	Consider the case $\lambda > 1/2.$ By treating the sum as a right Riemann sum, we have
	\[
	\sum_{k=n'}^{n} (k + 2)^{2\lambda - 2}\leq (n + 3)^{2\lambda - 1}/(2\lambda - 1) \enspace.
	\]
	As $(n + 3) \leq 2(n + 2)$ and $(n + 2) \geq (n + 1),$ we have
	\[
	\sum_{k = n'}^{n} \beta_{k, n}^2 \leq \cMb^2  \frac{ \Ro^2(n_0)}{n + 1} \enspace.
	\]
	
	Now consider the case $\lambda < 1/2.$ Again treating the sum as a right Riemann sum,  we have
	\[
	\sum_{k = n'}^{n} (k + 2)^{2 \lambda - 2} \leq \frac{1}{(1 - 2 \lambda) [n' + 1]^{1 - 2 \lambda}} \enspace.
	\]
	As $(n + 2) \geq (n + 1),$ it follows that
	\[
	\sum_{k = n'}^{n} \beta_{k,n }^2 \leq \cMb^2  \frac{\Ro^2(n_0)}{[n' + 1]^{1 - 2\lambda} (n + 1)^{2\lambda}} \enspace.
	\]
	Substituting $\sum_{k = n_0}^{n} \beta_{k, n}^2$ bounds in \eqref{eqn:IntConcBd}, the desired result is easy to see.
\end{proof}

\subsection{Conditional Results on the Bad Events}

On the first ``bad'' event $\Emid,$ the TD(0) iterate $\theta_n$ for at least one $n$ between $n_0 + 1$ and $n_0 + n_1$ leaves the $2 \Ro(n_0)$ ball around $\thS.$ Lemma~\ref{lem: bound on Emid} shows that this event has low probability. Its proof is the following.

\begin{proof}[Proof of Lemma~\ref{lem: bound on Emid}]
	From Corollary~\ref{cor:ODE perturbation}, we have
	\begin{equation*}
	\|\theta_{n + 1} - \thS\| \leq \|\theta(\tI{n + 1}, \tI{n_0}, \theta_{n_0}) - \thS\|
	+ \|\eDb\| + \|\eMb\| \enspace.
	\end{equation*}
	Suppose the event $\Gnd$ holds. Then from  \eqref{eq: norm of thetan - thetastar mon dec},
	\[
	\|\theta(\tI{n + 1}, \tI{n_0}, \theta_{n_0}) - \thS\| \leq \|\theta_{n_0} - \thS\| \leq \Ro(n_0) \enspace.
	\]
	Also, as $n_0 \geq \tfrac{K_\lambda6 \|A\| (\|A\| + 2\Km)}{\lambda}$, by Lemma~\ref{lem:SmallDE}, $\|\eDb\| \leq \Ro(n_0)/3.$
	From all of the above, we have
	\begin{equation*}
	\{\Gnd, \|\theta_{n + 1} - \thS\| > 2\Ro(n_0)\} \subseteq  \{\Gnd, \|\eMb\| > \Ro(n_0)/2\} \enspace.
	\end{equation*}
	From this, we get
	\begin{align*}
	\Emid \subseteq & \bigcup_{n = n_0}^{n_0 + n_1 - 1} \! \left\{\Gnd, \|\eMb\| > \tfrac{\Ro(n_0)}{2} \right\} \\
	\subseteq & \bigcup_{n = n_0}^{\infty} \! \left\{\Gnd, \|\eMb\| > \tfrac{\Ro(n_0)}{2} \right\} \enspace.
	\end{align*}
	Consequently,
	\begin{equation}
	\label{eqn:IntBd}
	\Pr\{\Emid\} \leq \\ \sum_{n = n_0}^{\infty} \! \!\Pr\!\left\{\Gnd, \|\eMb\| \! > \! \tfrac{\Ro(n_0)}{2} \right\}.
	\end{equation}
	
	Consider the case $\lambda > 1/2.$ Lemma~\ref{lem:MartConc} shows that
	\begin{equation*}
	\Pr\left\{\Gnd, \|\eMb\| > \tfrac{\Ro(n_0)}{2}  \right\}  \leq  2d^2 \exp\left[-\frac{n + 1}{8 d^3 \cMb^2 }\right] \enspace.
	\end{equation*}
	Substituting this in \eqref{eqn:IntBd} and treating the resulting expression as a right Riemann sum, the desired result is easy to see.
	
	Now consider the case $\lambda < 1/2.$ From Lemma~\ref{lem:MartConc}, we get
	\begin{equation*}
	\Pr\left\{\Gnd, \|\eMb\| > \tfrac{\Ro(n_0)}{2}  \right\}  \leq
	2d^2 \exp\left[-\frac{(n_0 + 1)^{1 - 2\lambda} (n + 1)^{2 \lambda}}{8d^3 \cMb^2 }\right] \enspace.
	\end{equation*}
	Let $\ell_{n_0} := (n_0 + 1)^{1 - 2\lambda} / 8d^3 \cMb^2.$ Observe that
	\begin{align}
	& \sum_{n = n_0}^{\infty} \exp[- \ell_{n_0} (n + 1)^{2\lambda}] \notag\\
	& \leq  \sum_{i = \lfloor (n_0 + 1)^{2\lambda}\rfloor }^{\infty} e^{-i \ell_{n_0}} |\{n :  \lfloor (n + 1)^{{2\lambda}}\rfloor = i\} | \notag\\
	& \leq  \frac{1}{2\lambda}  \sum_{i = \lfloor (n_0 + 1)^{2\lambda}\rfloor }^{\infty} e^{-i \ell_{n_0}} \; (i + 1)^{\frac{1- {2\lambda}}{{2\lambda}}} \label{eqn:cardUB} \\
	& \leq  \frac{1}{{2\lambda}}  \sum_{i = \lfloor (n_0 + 1)^{2\lambda}\rfloor }^{\infty} e^{-i \ell_{n_0}/2}  e^{-i \ell_{n_0}/2}\; (i + 1)^{\frac{1- {2\lambda}}{{2\lambda}}} \notag\\
	& \leq  \frac{1}{{2\lambda}} \left[\frac{( 1- {2\lambda})}{\ell_{n_0} {\lambda}}\right]^{\frac{1 - {2\lambda}}{{2\lambda}}} \! \! e^{\frac{1}{2}[\ell_{n_0}-\frac{1-2\lambda}{\lambda}]} \hspace{-1em} \sum_{i = \lfloor (n_0 + 1)^{2\lambda}\rfloor }^{\infty} \! \! e^{-i \ell_{n_0}/2} \label{eqn:MaxEst}\\
	& \leq  \frac{1}{{\ell_{n_0}\lambda}} \left[\frac{( 1- {2\lambda})}{\ell_{n_0} {\lambda}}\right]^{\frac{1 - {2\lambda}}{{2\lambda}}} \! \! e^{\frac{1}{2}[\ell_{n_0}-\frac{1-2\lambda}{\lambda}]}  e^{-\frac{\ell_{n_0}{n_0}^{2\lambda}}{4}} \label{eqn:RightRSum}\\
	& \leq \left[\frac{1- 2\lambda}{e}\right]^{\frac{1 - 2\lambda}{2 \lambda}} \left[\frac{8d^3 \cMb^2}{\lambda}\right]^{\frac{1}{2\lambda}}  \frac{\exp[- \frac{n_0}{64d^3 \cMb^2}]}{(n_0 + 1)^{\frac{1 - 2 \lambda}{2 \lambda}}} \label{eqn:n0Subs} \enspace \\
	& \leq \left[\frac{8d^3 \cMb^2}{ \lambda}\right]^{\frac{1}{2\lambda}}  \frac{\exp[- \frac{n_0}{64d^3 \cMb^2}]}{(n_0 + 1)^{\frac{1 - 2 \lambda}{2 \lambda}}} \label{eqn:SimConstants} \enspace.
	\end{align}
	The relation \eqref{eqn:cardUB} follows, as by calculus,
	\[
	|\{n :  \lfloor (n + 1)^{{2\lambda}}\rfloor = i\} | \leq \frac{1}{2\lambda}(i + 1)^{\frac{1- {2\lambda}}{{2\lambda}}} \enspace,
	\]
	\eqref{eqn:MaxEst} holds since, again by calculus,
	\[
	\max_{i \geq 0} e^{ - i \ell_{n_0}/2} (i + 1)^{\frac{1 - 2 \lambda}{2 \lambda}}
	\! \leq \! \left[\frac{( 1- {2\lambda})}{\ell_{n_0} {\lambda}}\right]^{\frac{1 - {2\lambda}}{{2\lambda}}} \! \! \! e^{\frac{1}{2}[\ell_{n_0}-\frac{1-2\lambda}{\lambda}]} \enspace,
	\]
	\eqref{eqn:RightRSum} follows by treating the sum as a right Riemann sum, \eqref{eqn:n0Subs} follows by substituting the value of $\ell_{n_0}$ and using the fact that $n_0^{2\lambda} \geq 4$ and \eqref{eqn:SimConstants} holds since $1 - 2 \lambda \leq 1.$ Substituting \eqref{eqn:SimConstants} in \eqref{eqn:IntBd}, the desired result follows.
\end{proof}

On the second ``bad'' event $\Eaft,$ the TD(0) iterate $\theta_n$ for at least one $n > n_0 + n_1$ lies outside the $\min\{\ei, 2 \Ro(n_0)\}$ radius ball around $\thS.$ Lemma~\ref{lem: bound on Eaft} shows that this event also has low probability.

\begin{proof}[Proof of Lemma~\ref{lem: bound on Eaft}]
	Assume the event $\Gnd$ holds for some $n \geq \nMid$.
	Then
	\[
	\|\theta_{\nMid} - \thS\| \leq 2\Ro(n_0).
	\]
	Hence from \eqref{eq:ODE_traj} and \eqref{eq:expMatBd}, for $t \geq \tI{\nMid},$ we have
	\begin{equation}
	\label{eqn:ODEDecayRt}
	\|\theta(t, \tI{\nMid}, \theta_{\nMid}) - \thS\| \leq \Kl e^{-\lambda (t - \tI{\nMid})} 2\Ro(n_0) \enspace.
	\end{equation}
	%
	

	%
	Now as $n_1 \geq (\nMid + 1) \left[ \frac{6\Kl \Ro(n_0)}{\ei} \right]^{1/\lambda} - n_0,$ it follows that $\forall n \geq n_0 + n_1,$
	\[
	\|\theta(\tI{n + 1}, \tI{\nMid}, \theta_{\nMid}) - \thS\| \leq \frac{\epsilon}{3} \enspace.
	\]
	Also, as $\nMid \geq \left(1 + \tfrac{K_\lambda6  \|A\| (\|A\| + 2\Km) C_* R_0}{\lambda \min\{ \epsilon, \Ro(n_0)\}} \right)(n_0 + 1),$ from Lemma~\ref{lem:SmallDE}, we have
	%
	%
	$\|E_{[\nMid,n+1]}^{\textrm d}\| \leq \epsilon/3$ for all $n \geq \nMid.$
	Combining these with Corollary~\ref{cor:ODE perturbation}, it follows that $\forall n \geq n_0 + n_1,$
	\begin{eqnarray*}
		\left\{\Gnd, \|\theta_{n + 1} - \thS\| > \min\{\epsilon, 2\Ro(n_0) \} \right\} & \subseteq &
		\{\Gnd, \|\theta_{n + 1} - \thS\| > \epsilon \}\\
		& \subseteq & \{\Gnd, \|E_{[\nMid,n+1]}^{\textrm m}\| \geq \tfrac{\epsilon}{3}\} \enspace.
	\end{eqnarray*}
	Hence from the definition of $\Eaft,$
	\begin{equation}
	\label{eqn:IntEaftBd}
	\Pr\{\Eaft\} \leq  \sum_{n = n_0 + n_1}^{\infty} \Pr\left\{\Gnd,  \|E_{[\nMid,n+1]}^{\textrm m}\| \geq \tfrac{\epsilon}{3}  \right\} \enspace.
	\end{equation}
	
	Consider the case $\lambda > 1/2.$ Lemma~\ref{lem:MartConc} and the definition of $\Ro(n_0)$ in Theorem~\ref{lem:WorstCaseThetaBd} shows that
	\begin{equation*}
	\Pr\left\{\Gnd,  \|E_{[\nMid,n+1]}^{\textrm m}\| \geq \tfrac{\epsilon}{3}  \right\}
	\leq
	2d^2 \exp\left[- \frac{(n_0 + 1)^{-2}(n + 1) \epsilon^2}{ 18 d^3 \cMb^2  \cS^2 R_0^2}\right] \enspace.
	\end{equation*}
	Using this in \eqref{eqn:IntEaftBd} and treating the resulting expression as a right Riemann sum, we get
	\begin{align*}
	\Pr\{\Eaft \}
	\leq  36 d^5 \cMb^2  \left[\frac{\Ro(n_0)}{\ei} \right]^2
	\exp \left[- \frac{(n_0 + n_1) \epsilon^2}{ 18 d^3 \cMb^2 [\Ro(n_0)]^2}\right] \enspace.
	\end{align*}
	Substituting the given relation between $n_1$ and $\nMid,$ the desired result is easy to see.
	
	Consider the case $\lambda < 1/2.$ From Lemma~\ref{lem:MartConc} and the definition of $\Ro(n_0)$ in Theorem~\ref{lem:WorstCaseThetaBd}, we have
	\begin{equation*}
	\Pr\left\{\Gnd,  \|E_{[\nMid,n+1]}^{m}\| \geq \tfrac{\epsilon}{3} \right\} \\
	\leq
	2d^2 \exp\left[- \frac{(\nMid + 1)^{1- 2 \lambda} (n + 1)^{2\lambda} \epsilon^2}{  18 d^3 \cMb^2  [\Ro(n_0)]^2}\right] \enspace.
	\end{equation*}
	Let $k_{\nMid} := \ei^2(\nMid + 1)^{1-2\lambda}/(  18 d^3 \cMb^2 [\Ro(n_0)]^2  ).$
	%
	\begin{equation*}
	\Pr\left\{\Gnd,  \|E_{[\nMid,n+1]}^{m}\| \geq \tfrac{\epsilon}{3} \right\} \\
	\leq 2d^2 \exp\left[- k_{\nMid} \,  (n + 1)^{2\lambda}\right] \enspace.
	\end{equation*}
	Then by the same technique that we use to obtain \eqref{eqn:RightRSum} in the proof  for Lemma~\ref{lem: bound on Emid}, we have
	\begin{align}
	&\sum_{n = n_0+n_1}^{\infty} \exp[- k_{\nMid} (n + 1)^{2\lambda}] \notag \\
	& \leq  \frac{1}{{k_{\nMid}\lambda}} \left[\frac{( 1- {2\lambda})}{k_{\nMid} {\lambda}}\right]^{\frac{1 - {2\lambda}}{{2\lambda}}} \! \! e^{\frac{1}{2}[k_{\nMid}-\frac{1-2\lambda}{\lambda}]}  e^{-\frac{k_{\nMid}{(n_0+n_1)}^{2\lambda}}{4}} \notag \\
	& \leq \left[\frac{1}{k_{\nMid} \lambda}\right]^{\frac{1}{2\lambda}} e^{-\frac{k_{\nMid}{(n_0+n_1)}^{2\lambda}}{8}} \notag \\
	& =  \left[ \frac{  18 d^3 \cMb^2   [\Ro(n_0)]^2}{\ei^2 \lambda (\nMid + 1)^{1 - 2\lambda} } \right]^{\frac{1}{2\lambda}} \exp\left[-\frac{\epsilon^2 (\nMid + 1)^{1-2 \lambda}(n_0 + n_1)^{2\lambda}}{ 144 d^3 \cMb^2   [\Ro(n_0)]^2 }\right] \notag
	%
	\end{align}
	where the second inequality is obtained using the facts that $(n_0 + n_1)^{2\lambda} \geq n_0^{2\lambda} \geq 4$ and $1 - 2\lambda \leq 1$ and the last equality is obtained by
	substituting the value of $k_{\nMid}.$ From this, after substituting the given relation between $n_c$ and $n_1,$ the desired result is easy to see.
\end{proof}

\subsection{Detailed Calculations for the Proof of Theorem~\ref{thm: convergence rate}}
We conclude by providing all detailed calculations for our main result, Theorem~\ref{thm: convergence rate}. 

	From Lemma~\ref{lem: decomposition}, by a union bound,
	\[
	\Pr\{\cE^c(n_0, n_1)\} \leq \Pr\{\Emid \} + \Pr\{\Eaft\} \enspace.
	\]
	We now show how to set $n_0$ and $n_1$ so that each of the two terms above is less than $\delta/2.$
	
	Consider the case $\lambda > 1/2.$ Let
	\begin{align}
	N_0(\delta)
	=
	\max\left\{
	\tfrac{K_\lambda6 \|A\| (\|A\| + 2\Km)}{\lambda},
	2^{\frac{1}{\lambda}},
	8d^3\cMb^2\ln\left[\tfrac{ 32d^5\cMb^2}{\delta}\right]
	\right\}
	=& O\left(\ln\tfrac{1}{\delta}\right) \label{eq:N0 order}
	\enspace,
	\end{align}
	\begin{align*}
	N_c(\ei, \delta, n_0) =
	& \max \Bigg\{\left[\left(1  + \tfrac{K_\lambda6  \|A\| (\|A\| + 2\Km)}{\lambda \min\{ \epsilon, \Ro(n_0)\}} \right)\Ro(n_0)\right], \\
	& \hspace{4em} \frac{18 d^3 \cMb^2}{(6\Kl)^{1/\lambda}} \left[\frac{\Ro(n_0)}{\ei}\right]^{2 - \tfrac{1}{\lambda}} \ln\left[ 72 d^5 \cMb^2 \left[\frac{1}{\delta} \right]\left[\frac{\Ro(n_0)}{\ei}\right]^2\right]\Bigg\},
	\end{align*}
	so that $N_c(\ei,\delta,N_0(\delta)) = \tilde{O}\left(\max\left\{\frac{1}{\ei} \ln\left[\frac{1}{\delta}\right], \left[\frac{1}{\ei}\right]^{2- \tfrac{1}{\lambda}} \left[\ln \frac{1}{\delta}\right]^{3 - \tfrac{1}{\lambda}}\right\}\right),$
	and let
	\begin{align*}
	N_1(\ei, \nMid, n_0)
	= (\nMid + 1) \left[ \frac{6\Kl \Ro(n_0)}{\ei} \right]^{1/\lambda} - n_0,
	\end{align*}%
	so that
	\begin{equation}
	\label{eqn:ordN1}
	N_1(\ei, N_c(\ei, \delta, N_0(\delta)), N_0(\delta)) = \tilde{O}\left(\max\left\{\left[\frac{1}{\ei}\right]^{1+ \frac{1}{\lambda}} \left[\ln\frac{1}{\delta}\right]^{1 + \frac{1}{\lambda}}, \left[\frac{1}{\ei}\right]^{2} \left[\ln \frac{1}{\delta}\right]^{3}\right\}\right).
	\end{equation}
	
	%
	%
	Let $n_0 \geq N_0(\delta),$ $\nMid \geq N_c(\ei, \delta, n_0)$ and $n_1 \geq N_1(\ei, \nMid, n_0).$ Then from Lemma~\ref{lem: bound on Emid}, $\Pr\{\Emid \} \leq \delta/2$ and from Lemma~\ref{lem: bound on Eaft}, $\Pr\{\Eaft \} \leq \delta/2.$ Hence $\Pr\{\cE^c(n_0, n_1)\} \leq \delta.$
	Consequently,
	\(
	N(\epsilon,\delta) = N_1(\ei, N_c(\ei, \delta, N_0(\delta)), N_0(\delta))
	\)
	%
	satisfies the desired properties, which completes the proof for $\lambda > 1/2.$
	
	Now consider the case $\lambda < 1/2.$
	The same exact proof can be repeated, with the following $N_0$, $N_c$ and $N_1.$
	\begin{align}
	N_0(\delta)
	=
	\max\left\{
	\tfrac{K_\lambda6 \|A\|(\|A\| + 2\Km)}{\lambda},
	2^{\frac{1}{\lambda}}, \tfrac{64d^3 \cMb^2}{2\lambda} \ln\left(\tfrac{32 d^5 \cMb^2}{\delta \lambda}\right)
	\right\}
	=& O\left( \ln\tfrac{1}{\delta} \right) \enspace, \label{eq: small lambda O}
	\end{align}
	\begin{align*}
	N_c(\ei, \delta, n_0) = & \max \Bigg\{\left[\left(1  + \tfrac{K_\lambda6  \|A\| (\|A\| + 2\Km)}{\lambda \min\{ \epsilon, \Ro(n_0)\}} \right)\Ro(n_0)\right], \\
	& \hspace{4em} \frac{4 d^3 \cMb^2 }{2 \lambda K_\lambda^2 } \ln\left( \frac{72 d^5 \cMb^2 }{\lambda}  \left[\frac{1}{\delta}\right] \frac{  [\Ro(n_0)]^2}{\ei^2 } \right) \Bigg\},
	\end{align*}
	so that $N_c(\ei, \delta, N_0(\delta)) = \tilde{O}\left(\frac{1}{\ei} \ln \frac{1}{\delta}\right) $ and let
	\begin{equation}
	\label{eq: small lambda tilde O}
	N_1(\ei, \nMid, n_0)
	= (\nMid + 1) \left[ \frac{6\Kl \Ro(n_0)}{\ei} \right]^{1/\lambda} - n_0,
	\end{equation}
	so that $N_1(\ei, N_c(\ei, \delta, N_0(\delta)), N_0(\delta)) = \tilde{O}\left(\big[{(1/\epsilon)}\ln{(1/\delta)}\big]^{1+{1/\lambda}}\right).$
	%
	%
	%
	Thus
	\(
	N(\epsilon,\delta)
	=
	N_1(\ei, N_c(\ei, \delta, N_0(\delta)), N_0(\delta))
	\)
	satisfies the desired properties for the case $\lambda<1/2$.

	For $\lambda = 1/2$, the same process can be repeated, resulting in the same $O$ and $\tilde{O}$ results as in \eqref{eq: small lambda O} and \eqref{eq: small lambda tilde O}.

%
%

\end{document}